\documentclass[notitlepage,a4paper]{report}
\usepackage{graphicx}
\usepackage{multirow}
\usepackage{subfig}
\usepackage{amsmath}
\usepackage{amssymb}
\usepackage{amsthm}
\usepackage{relsize}
\usepackage{ulem}
\usepackage{xcolor}
\usepackage[unicode=true, bookmarks=true,bookmarksnumbered=false,
            bookmarksopen=false,hyperfootnotes=false,breaklinks=true,
            pdfborder={0 0 1},backref=false,colorlinks=false]{hyperref}
\usepackage{fancybox}

\newcommand\hlight{\bgroup\markoverwith
  {\textcolor{yellow}{\rule[-.5ex]{2pt}{2.5ex}}}\ULon}

\newtheorem{lem}{Lemma}
\newtheorem{thm}{Theorem}

\newtheorem{defn}{Definition}

\title{An Efficient Approach to Communication-aware Path Planning for Long-range Surveillance
Missions undertaken by UAVs}
\author{Hrishikesh Sharma and Tom Sebastian}
\date{\today}

\makeatletter
\def\@maketitle{
\begin{center}
{\Huge \bfseries \sffamily \@title \par}~\\[32ex] 
{\Large \textit{Technical Report by}}\\[8ex]
{\Large  \@author}\\[16ex] 
\includegraphics[scale=0.3]{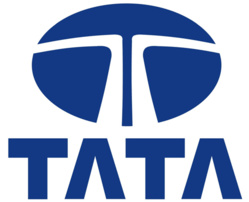}\\[8ex]
\@date\\[8ex]
\end{center}
}
\makeatother

\begin{document}

\maketitle
\thispagestyle{empty}
\newpage

\setcounter{page}{1} %Start the actually document on page 1

\begin{abstract}
While using drones (UAVs) for remote surveillance missions, it is mandatory
to do path planning of the vehicle since these are pilot-less vehicles. Path
planning, whether offline (more popular) or online, entails setting up the
path as a sequence of locations in the 3D Euclidean space, whose
coordinates happen to be latitude, longitude and altitude. For the specific
application of remote surveillance of long linear infrastructures in
non-urban terrain, the continuous
3D-ESP (Euclidean Shortest path) problem practically entails two important scalar
costs. The first scalar cost is the distance travelled along the planned
path. Since drones are battery operated, hence it is needed that the path length between fixed
start and goal locations of a mission should be minimal at all costs. The
other scalar cost is the cost of transmitting the acquired video during the
mission of remote surveillance, via a camera mounted in the drone's belly.
Because of the length of surveillance target which is \textit{long} linear
infrastructure, the amount of video generated is very high and cannot be
generally stored in its entirety, on board.
If the connectivity is poor along certain segments of a naive path, to boost video
transmission rate, the transmission power of the signal is kept high, which
in turn dissipates more battery energy. Hence a path is desired that
simultaneously also betters what is known as ``communication" cost.  These
two costs trade-off, and hence Pareto optimization is needed for this 3D
biobjective Euclidean shortest path problem. In this report, we study the
mono-objective offline path planning problem, based on the distance cost, while posing the
communication cost as an upper-bounded constraint. The biobjective path planning solution
is presented separately in a companion report.
\end{abstract}

\tableofcontents
\listoffigures

\chapter{Introduction}
In recent times, there has been a great demand of UAVs in numerous
civil industrial and military operations around the world. UAVs are often
deployed for missions that are too \underline{"dull, dirty, or dangerous"}
for manned aircraft\cite{cov_pplan_pap}. One of such civilian application
of UAVs corresponds to remote surveillance of long linear infrastructures.
Such (utility) infrastructures typically include power grids, oil and gas
pipelines, railway corridors etc. The area of surveillance for such targets
is typically vast, running into tens of kilometers. Hence mission planning
for such applications has received separate attention. Other than vastness,
such infrastructures also tend to run through \textbf{complex} remote
terrains, such as forests, hilly regions, wetlands etc, some of
which can be no-fly zones. Such terrains are
mostly non-urban. Hence remote surveillance using UAVs has emerged as a
cost-effective and efficient solution for monitoring of such
infrastructures \cite{ncc_pap}.

\begin{figure}[!h]
\begin{center}
\includegraphics[scale=.6,keepaspectratio=true]{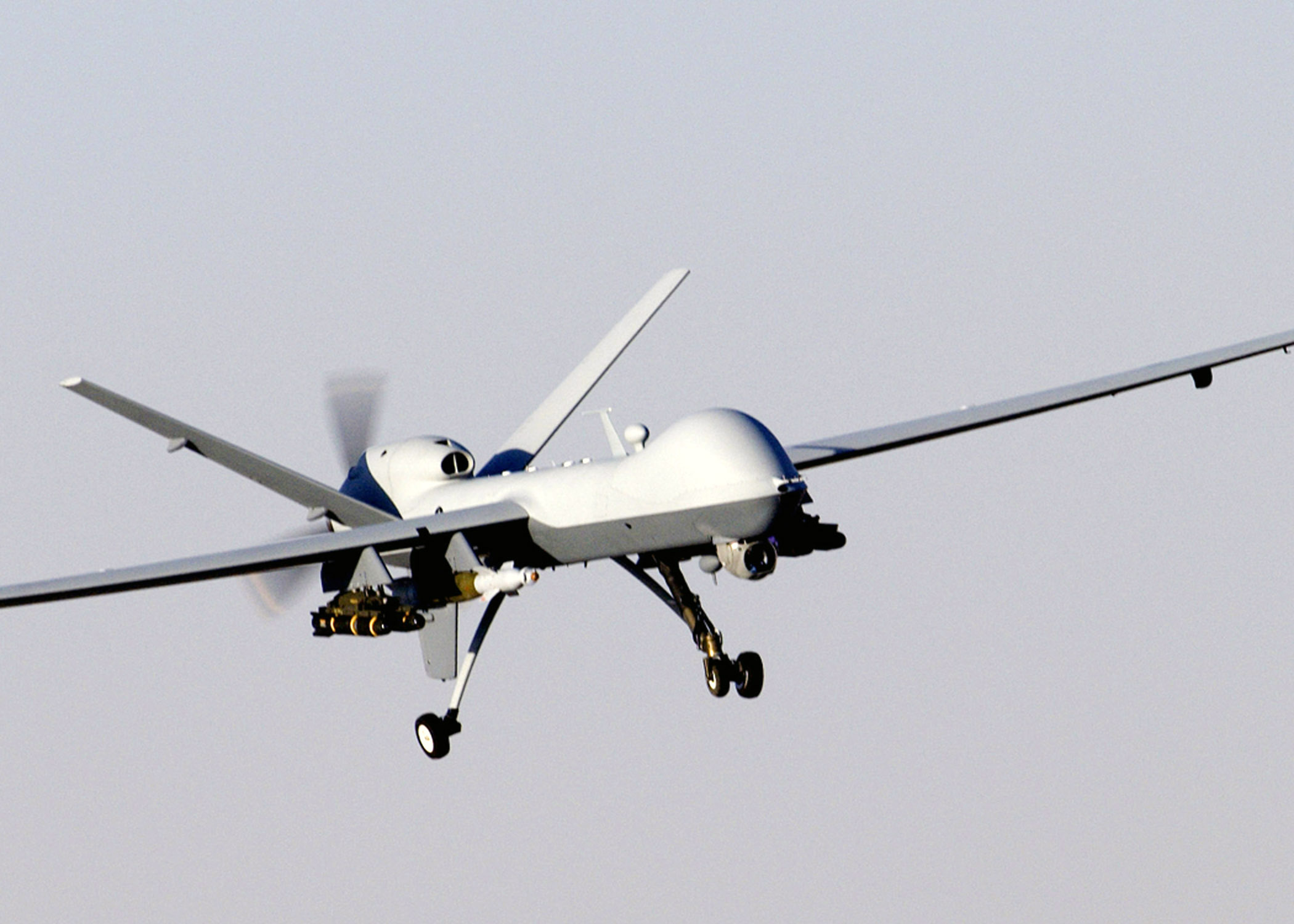}
\end{center}
\caption{An Illustrative Drone in Operation}
\label{drone_fig}
%\vspace{-.23in}
\end{figure}

\section{Requirement for Live Video Communication}
\label{live_sec}
It is imperative that surveillance mission of linear infrastructures
results in acquisition of very long and bulky videos. For lack of unlimited
storage on-board, or in case of requirement of quick response such as
supply breakdown, it is important to transmit the video along the flight.
More precisely, often the surveillance video or the sensor data gathered
must be transmitted continuously from the UAV to one or more ground control
stations (GCS). In disaster and breakdown management, real-time sensor data
acquisition and analytics is required to be carried out, before the damage
becomes widespread. Further, on long duration flights, which are now quite
practical \cite{solar_uav}, one cannot store entire high volume of video
data, typically running in terabytes, and hence needs to transmit it out most of the while. As this data can
include urgent high-volume data such as live video, especially in case of
long linear infrastructures, maximally uninterrupted high-quality (SNR)
connectivity downlink from the UAV to GCS is often required. The required
uninterrupted connectivity in remote rural terrain, which may pass through
forests, is practically never there. In fact, availability of a carrier by
whichever means (RF/Wi-Fi/Public Cellular) is known to be sparse and hence
intermittent, in lack of sufficient number of transmitters installed and
active in such terrains.  Such connectivity can be provided using satellite
links, but data rate of such links is well-known to be slow. Hence for such
surveillance mission, it makes a lot of sense to perform path planning
that additionally factors in the various degrees of available signal
coverage, so as to maximize the amount of video transmitted.

\section{Path Planning as an Optimization Problem}

\shadowbox{
\begin{minipage}{\textwidth}
\begin{defn}
A path planning algorithm is one that
{\it{\textbf{\underline{computes a trajectory}}}} from the UAV's present
location to a desired future location, e.g. an intermediate
or final target\cite{ppu}. The computation is such that it optimizes
certain cost of the path so evolved.
\end{defn}
\end{minipage}
}

In general, in any space, very large number of solutions/paths exist
between a source location and a target/destination location. Sometimes the
number of such paths can be countably infinite. However, in reality, we
tend to choose paths that suit our specific requirements. These
requirements manifest themselves as optimization criteria, entailing
certain costs. Hence any such computation of path is done so as to balance
various costs attached with executing the mission, e.g. battery power,
total distance traveled etc.

Coming to the case of UAVs, path planning of a UAV almost always boils
down to optimizing the path length. This is because UAVs are battery
operated. Hence, especially for long linear infrastructures, it has always
made sense to have a path that can provide maximum-distance/maximum-length
coverage of the near-linear/piecewise-linear infrastructural corridor. Such
corridors typically have length which can be many
times multiple of the maximum path length that a specific battery for a
specific class of UAV can offer.

However, as argued earlier, for certain applications, one also needs to
maximize connectivity for communication along the same path. This is
especially true for the application of remote surveillance of long linear
infrastructures. This suggests that the path planning problem could be cast
as a \textbf{biobjective} optimization problem, in 3D Euclidean space, for
UAVs.

Still, if the path planning is done in a \textbf{naive}, mono-objective
fashion for such applications, then it can happen that the UAV may end-up
at big-enough coverage hole area. To be able to still ensure that storage
for collected/sensed data does not overflow, the UAV will then need to find
alternative route, in an online fashion. It is a known fact that any online
path planning in a space whose topology is static yields poorer
results/cost than offline planning. Thus a naive approach is bound to
result in longer mission time. It will also waste valuable battery
resource, since algorithm execution consumes power at runtime.
Even further, the slowness of algorithm may not allow UAV to fly
at its peak speed, since next step has to be decided first before the step
is taken by the UAV. Finally, the solution will be poorer since the solution will be locally
optimal (finds a detour near obstacle only when the vehicle is close to
it), not globally optimal as off-line one. Since the topology of the solution design space is almost completely known
a-priori (no-fly zones/obstacles, wireless coverage holes, terrain,
surveillance object's location etc.), off-line design is not only equally
possible, but also avoids the above disadvantages of the online algorithm.

On the other hand, if the path planning is done such that both objectives
are satisfied to the extent possible together, then in best case, one may
be able to get all-time connectivity along the path. If not, even then, one
can get good connectivity for most length along the path, which can lead to
effective transmission rates most of the while \cite{ghaffar_2014_pap}.
This not only leads to efficient resource utilization (even signal
transmission power is derived from battery power), but also in case of
emergency, the report-response time can be minimized.

In this technical report, we provide the design details of an offline
heuristic algorithm for path planning with such communication requirement.
The path planning problem is tailored for the required surveillance task,
and hence favors its deployment on fixed-wing UAVs. However, as can be seen
later on, by trivializing the maximum banking angle and minimum route leg
length constraints, the same algorithm can also be deployed on commonplace
rotary-wing UAVs as well.

\chapter{Prior Literature}

The prior literature in last few decades till date has mostly focused on
traditional path planning. Further, this planning is either agnostic of
terrain, or has concerned mainly with Urban areas. Such literature limits
itself to consideration of distance covered as cost.

In the plethora of prior works on communication-agnostic path/route planning, A* has been the popular workhorse algorithm, in a
discrete navigation space \cite{goerzen2009survey}. However, especially for route planning with turning
constraints, A* is known to be inefficient \cite{victor_thesis}.
In another
solution \cite{turn_transform_pap}, a graph transformation is performed first,
followed by a path search on the transformed graph. But such algorithm is
inefficient for simple reason: in a 2D grid for example, there are
typically at maximum 1 neighbor left post transform for each node, since
the turn limit is generally about 30$^{\circ}$ \cite{turn_2010_pap},
\cite{angle_pplan_pap}, \cite{haz_pap}. On such a sparse graph, the
discretization
offset/error is very high. \cite{haz_pap}
solves this problem using \textit{Bellman-Ford} algorithm, which is known
to have a much higher complexity (O($|$E$|\cdot|$V$|$). In yet another
related work \cite{windy_plan_pap}, angle-constrained motion planning is
carried out along with other kinematic constraints arising out of high
windy conditions. However, due to dynamic change in the lift, the turn
limit itself is a variable. Towards minimization of the
\textit{alternative} aspect of number of turns, few solutions are provided in 
\cite{fewer_turns_pap}, \cite{cov_pplan_pap}. However, it is shown 
that such paths can be up to 22\% longer than ideal
solution on an average, even in 2D case.

The communication cost has in past been sometimes considered in path
planning problems that relate to robotic path/motion planning. In all these
works \cite{ghaffar_2011_pap}, \cite{ghaffar_2011_motion_pap},
\cite{adaptive_muav_pap}, \cite{zavlanos_pap}, \cite{mostofi_2008_pap},
\cite{ghaffar_2009_pap}, \cite{ghaffar_2010_pap}, \cite{sigstren_est_pap},
this cost component is modeled as the \textbf{receiver-end} probability of
connectivity of the nodes to the base station at all the times, which has
to be maximized.  Hence this cost is expressed as random variables arising
in physical layer function: estimated packet error rate, estimated queue
size, at various moments of time. Almost all these works consider multiple
robots in action, rather than a single robot, which further involves robots
themselves acting as relays. In a single machine case that of a single
surveillance UAV, there is no possibility of relaying.

The problem with receiver-end modeling is that the effect of communication
channel in between has also to be factored. The wireless communication
channel is inherently non-deterministic in terms of its gain. Hence to
factor its effect on effective reception rate entails usage of
probabilistic, or more precisely, stochastic models. 

The main reason behind
high degree of uncertainty in reception rate while transmission happens
from a moving vehicle is because of multipath effect
\cite{ghaffar_2009_pap}.  Due to multipath, problem is at a short distance
from a sampling point in space, the SNR of sampled carrier signal may
differ widely \cite{magnus_thesis}. One possibility is to have a very find
sampling grid, so that the SNR variation appears gradual when viewed on a
discrete graph/grid. However, practically, one cannot do dense sampling for
the specific application of ours, given the actual size of the Euclidean space in
which path planning has to happen (tens/hundreds of kilometers).

Thankfully, for the application we are interested in, we do not require
full-fledged stochastic modeling of the channel. 
This is because of
non-urban, remote terrain, where multipath effects or small-scale fading is
almost non-existing. The other two fading components, namely the
medium-scale and large-scale fading, can be predicted using deterministic
signal propagation modeling, to a good degree of approximation
\cite{ncc_pap}. Online learning of signal quality maps
\cite{decentralized_mostofi_pap} is not worth additional online computation
and thus spending some more battery power, since with fixed
topology/locations of base stations, and fixed terrain, and quasi-static
number of users/quasi-static traffic pattern at a base station, channel
characteristics aren't expected to change heavily with time.

The work in \cite{zavlanos_pap} further considers a situation where the
information generation rate itself is dynamic and has to be modeled
separately, something that does not hold good in our case. Using rough
calculations, a fixed, average frame rate such as 30 fps has been found
good enough to capture videos that we require, with significant overlap of
views captured between two successive frames. Hence, in summary, from prior
work \cite{ncc_pap}, we assume a-priori presence of connectivity maps which
are assumed to be quasi-stationary.

Also, camera as a visual sensor differs from other sensors in the sense
that typically the distance from the object being imaged does not matter,
as long as the distance between object and the camera is lesser than
\textbf{camera field of view depth} \cite{3dcam_pap}. Hence approaches such
as \cite{hyun_pap} become irrelevant, which take distance between UAV and
surveillance object into account. Similarly irrelevant are approaches that
plan for moving targets, not just moving remote sensing platform i.e.
dynamic coverage \cite{mostofi_2014_pap}, \cite{ghaffar_central_pap}.

\chapter{Problem Restriction}
Most real life problems which are modeled and solved mathematically, are
first solved in a restricted sense. That is, certain part of details are
omitted for various reasons, and the rest of the problem is solved. The
details that are ignored must be such that the loss of accuracy of solution
can at least be shown to be restricted.

\section{Restricting Consideration of UAV Motion}
The prior work on UAV path planning is mostly organized in two classes.
The first class, which views the optimization problem as an optimal flight
control problem, retains the details pertaining to the vehicle's movement.
However, it typically abstracts out second-order and higher kinematical
effects from the modeling, via a model called \textit{Dubin's vehicle}
\cite{kaarthik_thesis}. This class of problem is known as the
\textbf{Motion Planning} problem. On the other hand, the second class
separates out the vehicle control and optimization concerns. The
corresponding optimization problem for the latter concern is known as the
\textbf{Route Planning} problem.

It has been provably brought out in \cite{ghaffar_2012_pap},
\cite{ghaffar_2014_pap} that the optimal velocity for each mobile agent is
the maximum possible velocity. Further, the cost of such mission is similar
to that of a slower-speed mission undertaken by the mobile agent. If we do
not exploit this fact, then, in a general setup, a path of any moving system in earth's
gravitational field is a \textbf{continuous space-time curve} obeying the
Newtonian laws of motion anyway. Space-time path optimization problems are
known to be hard. Hence we ignore time as a variable. but retain location
as a (continuous) variable, in the tractable model of path planning. Thus
the motion dynamics is ignored in our approach to solution, and we focus on
route planning rather. As such, assumption of such
constant (maximum) speed decouples the motion and route planning parts
of path planning. Similar philosophy is also
followed in the popular communication-agnostic 2-phase decoupled trajectory planning approach
\cite{goerzen2009survey,spline_pp_pap, 2ph_pap}, where the factoring of
dynamic constraints is pushed to $2^{nd}$ phase, while focussing on route
planning in $1^{st}$ phase. Hence we have currently avoided dealing with Dubin's path
design or motion primitive-based design, towards (direct) motion planning \cite{pp_icuas16_pap}.

Even when it comes to route planning, finding an optimal solution is NP-hard, due to presence of obstacles
\cite{3dhard_pap}. Hence two philosophical lines of solution exist:
combinatorial and sampling-based. Especially in motion planning, randomized
sampling based approaches are preferred \cite{spline_pp_pap},
whenever the solution space is 
high-dimensional. In the premises of our problem, however, the
turn-angle constraint being holonomic, the solution space is practically low-dimensional. Further, we impose five constraints in
the overall problem. Hence size of search space
becomes quite tractable, along with its low dimensionality. Hence
it is more prudent to follow the combinatorial, greedy approach
where optimality is not asymptotic, unlike randomized sampling-based
algorithms e.g. RRT$^*$, PRM$^*$, FMT$^*$ etc., as described below.

\section{Restriction against Continuous Optimization}
\label{discrete_opt_sec}
With regard to optimization problems, another important modeling decision
has to be taken before the problem model is specified. This decision
pertains to modeling the problem as a discrete or continuous optimization
problem.

A number of reasons have been cited in various relevant literature, which
in a majoritarian way, argue in favor of modeling the problem as a discrete
optimization problem. Some of the reasons are
\begin{itemize}
\item Even though the number of decision variables are less, the graph size
(or more precisely, the size of feasible region), is quite big, given the
physical vastness of surveillance area.

\item Even if the shortcomings described above could be overcome, any
continuous routing model that produces routes having smooth curves probably
produces routes that are unflyable by a human pilot or a human UAV
controller. Thus, \cite{mil_route_pap} concludes that continuous routing
models are unsuitable for use in autorouters.
\item  Most military applications use Digital Terrain Elevation Data (DTED)
as input to their route planners \cite{real_pp_pap}, \cite{3d_pplan_pap}.
DTED is grid-based in nature, and hence practically most route planners are
anyway grid-based.
\item The loss of accuracy in case of discretization is not very high. For
the popular route planning algorithm proposed in \cite{lazy_theta_pap}, the
shortest paths formed by the edges of 8-neighbor 2D grids are at maximum
$\approx$ 8\% longer than the shortest paths in the continuous environment.
In the same paper, it is further proven that the shortest paths formed by
the edges of 26-neighbor 3D grids can at maximum be $\approx$ 13\% longer
than the shortest paths in the continuous environment.
\end{itemize}

Discretization has a side effect/disadvantage that any concatenation of
straight segments, with only few different angles of possible movement to
neighboring cells, in a fixed-topology grid, does not fully utilize the
flight capabilities of the air vehicle. For example, in a
square/rectangular grid, the theoretically-allowed turns are only
restricted to either 45$^{\circ}$ or 90$^{\circ}$, to the left or right of
the current cell/location. While the shape of grid cell can be changed
\cite{cell_pap}, another possible option is to have a
\textit{random-topology} grid, which is non-uniform, generated by some
random process \cite{3drisk_pap}. A very popular non-randomized solution is
known as any-angle path planning in a non-random grid. Our solution, as
will be detailed later, is derived from that route planning strategy.

\chapter{Problem Modeling}
Formal specification of any optimization problem entails specification of
the decision variables, the cost/objective as a function of these
variables, and constraints wherever applicable \cite{nocedal_book}. For the
route planning problem, these parts are detailed as below.

\section{From Motion Planning to Route Planning}

The cameras that are practically used for aerial surveys either do not have
multiple frame rates, or cannot be configured dynamically. As such also,
there is no need to dynamically change frame rate, since there is
sufficient overlap of scene of the remote-surveyed object across successive
frames, and there isn't any need to speed up or slow down the frame rate
than typical frame rate either. Once frame rate is fixed, then the other
reason why a UAV must speed up or slow down would depend if the sensing and
transmitting rates differ (information generation vs information
communication rates). Our algorithm is specifically designed in a way that
it can handle large communication holes. Hence, once more, there is no
specific need to speed up or slow down of UAV during mission either. This
leads to constant-speed situation: in which case, one must follow just on
the route planning, rather than simultaneously also planning for motion
dynamics. Once it comes to this scenario, it is shown in
\cite{ghaffar_2014_pap} that the best possible UAV speed while
simultaneously optimizing for distance and communication costs is the
\textit{maximum possible speed} of the UAV.

If ever one has to do motion planning, then one may still consider
\cite{al_theta*_pap}, \cite{aal_theta*_pap} algorithms, which do not entail
first or higher order dynamical system of equations. Similar is the
algorithm Kinematic A* proposed in \cite{pplan_graph_pap}, which goes via
offline modeling of speed-dependent movement cost component between two
adjacent nodes of the graph. In fact, Dynamic A* algorithm proposed in
\cite{dy_star_pap} also follows same philosophy. \cite{turn_2010_pap} goes
a step further and provides motion planning in wake of a steady-state flow
field, typically modeling the wind.

\section{Decision Variables}
There is only one set of decision variables involved. This set is the 3D
\textit{instantaneous} location of the UAV: $\langle x,y,z\rangle$. If the altitude is
held/assumed constant, then the location is a 2-tuple, formed by latitude
and longitude in world coordinate system. This tuple influences various
cost components to be described later, including the sensing cost that we
have not considered in this work.

\section{Constraints}
Most of the constraints for the problem are fairly the same across
literature, and are rooted in physical control of the unmanned vehicle
\cite{3d_heur_pap}, \cite{sparse_a*_pap}, \cite{lowpenetration_pap}. We
also impose certain constraints on route planning. For more intuition about
the constraints, one may refer to \cite{real_pp_pap}.

\subsection{Minimum Route Leg Length}
\label{min_rleg_sec}
This constraint forces the route to be straight segment, for a
predetermined minimum distance, before initiating a turn. Aircraft traveling
long distances avoid turning constantly because it adds to pilot fatigue
and increases navigational errors. Such constraint generally occurs in
conjunction with the turn angle constraint described next.

\subsection{Maximum Turning Angle}
\label{max_tangle_sec}
This constraint forces the generated route to make turning manoeuvre, less
than or equal to a predetermined maximum turning angle. Such constraint is
meant for fixed-wing vehicles, which is what we have assumed for our
application. The actual kinematics behind this constraint is explained in
\cite{haz_pap}.

%Furthermore, these constraints need not be fixed, and may vary during the
%course of the mission (i.e., a shorter leg length may be needed at the end
%of a mission than at the start).

%The main idea of SAS is to integrate the constraints proposed by Robert J.
%Szczerba into the search process in order to reduce the search space and
%search time for A* algorithm. These constraints include minimum route leg
%length, maximum turning angles, route distance constraint and fixed
%approach vector to goal position.

\begin{figure*}[!h]
\begin{center}
\subfloat{\label{an_1}\includegraphics[scale=.66]{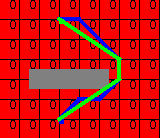}}
\qquad
\subfloat{\label{an_2}\includegraphics[scale=.66]{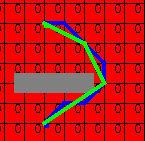}}
\caption{Illustrative Shortest Euclidean Paths with Angle Constraint. Green
Line Depicts the Finalized Path}
\label{angle_fig}
\end{center}
\end{figure*}

\subsection{Minimum Separation}
\label{min_sep_sec}
This is a very unique constraint for our application. All linear
infrastructures, across all nations, require that no man, machine or any
other artificial system ever intrude within a pre-specified vicinity of the
infrastructure. Such restricted area is called \textbf{Right of Way}. In
case of unmanned flight, flying too close to the infrastructure and sudden
control failure of the UAV (unmanned vehicles are fairly prone to sudden
death) can also mean that the vehicle or a part of it falls on the
infrastructure and damages it, which may lead to interruption of the
critical supply that that particular infrastructure bears. Further,
specifically in the case of power grid corridor, flying too close to HV
transmission lines leads to electromagnetic interference between the
electrical circuits of the vehicle, and the power line. Such interference
has led to corruption of internal flight control data over the internal
wires, and eventual failure/dropping dead of few pilot flights for us.

A closely related but separate constraint arises from the fact that at a
nominal speed of UAV, and a nominal video frame rate, there is significant
overlap between the length of linear infrastructure captured between any
two successive frames. This overlap will only reduce if the camera comes
closer to the infrastructure. However, the legally enforced constraint of
right of way according to us is more sacrosanct, since it relates
to system safety. Hence we do not take a
minimum of these two lower bounds, but just stick to the right of way
constraint only. Overall, from our view, it is of high importance that such
a constraint on route design be enforced.

\subsection{Maximum separation}
\label{max_sep_sec}
This is yet another unique constraint arising from the fact that a camera
mounted in the belly of the UAV is being used to perform the remote
surveillance task. Any visible range digital camera, using any technology
has a performance limit, fundamentally due to its usage of lens system, and
a finite-sized CCD/CMOS backplane. It is well-proven  \cite{3dcam_pap} that
if the distance between the object being imaged, and a digital
visible-range camera is more than a specific distance known as
\textbf{Depth of Field}, then the object cannot be properly imaged/will be
as hazy as underlying background, typically earth's surface/terrestrial
terrain. The application demands that the foreground object of
interest be reasonably sharp and distinct from background \cite{iciar_pap},
for computer vision techniques to be applied for segmenting out thinner
infrastructures against the vast and wide background.

A closely related
but separate constraint arises from the fact that certain foreground object
of interests e.g. power transmission lines are fairly thin. In such case,
to avoid they being imaged at sub-pixel granularity, the camera must stay
within a fixed distance from the infrastructure/object at all times.

In summary, once again,
from our view, it is of high importance that such a constraint on route
design be enforced.

\subsection{Storage Constraint}
\label{max_store_sec}
This is the third unique constraint arising due to the specific application
of ours. Since we are considering communication cost as a component, there
are expected to be regions which are no-coverage zones. These regions arise
since there is no base station in their vicinity. More specifically, in
practical scenarios, most wireless receivers cannot recover signal if the
SNR is too low (lower bound). Since there is a single receiver, and single
baseband processing unit involved (no MIMO case), SNR thresholding using
lower bound implies thresholding of Tx-Rx distance using upper bound
(ignoring proportionality constants). In turn, this implies that if closest
access point is situated at a distance beyond $d_{minSNR}$, then the UAV is
in \textbf{zero coverage} zone. Such a situation can also arise if
instead of omni-directional antennas, sector antennas are mounted on the
base station.

\begin{figure}[!h]
\begin{center}
\includegraphics[scale=.6,keepaspectratio=true]{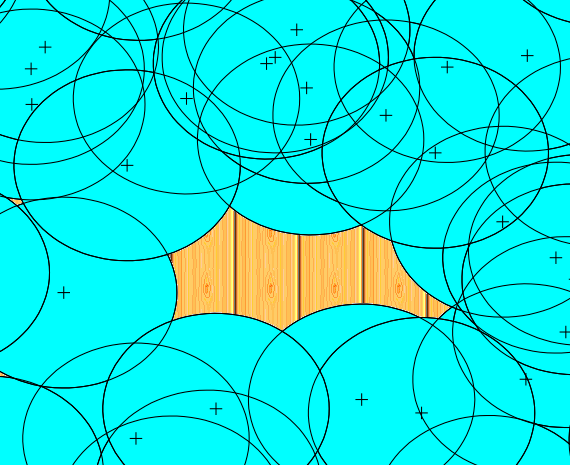}
\end{center}
\caption{A Coverage Hole in Presence of Multiple Receivers}
\label{hole_fig}
%\vspace{-.23in}
\end{figure}

It is not possible to altogether avoid these regions while designing an
optimal path. We are having two cost components as described later in this
chapter, and these components trade off with each other in a
\textit{Pareto} sense. Hence it is possible that while passing through a
zero-coverage zone, the path designed saves a bit on distance cost, while
losing a bit on communication cost.

While flying through such regions if the designed route leads to so, we
have discussed in section \ref{live_sec} that the continuous video being
shot has to be stored on-board. The stored video can be transmitted later,
when good SNR conditions prevail. However, the storage is not unlimited. If
the storage becomes full, then overflow will happen, resulting in a part of
video not being available later for any analysis. The filling will happen
since the continuous video acquisition leads to terabytes of data
generation, while the storage limit would typically be much smaller than
that. Since the application of
ours mostly deals with fault/damage identification, any missing part in
video capture is highly undesired \cite{dicta_pap}. In order to avoid
overflow while storing, one can force the path planning algorithm to only
consider options where the segment of the path lying in a specific region
has a length that is upper bounded by some constant. This constant
(distance) is decided based on the (static) UAV speed, the frame rate of
the video camera, the frame size, and the amount of on-board storage. For
the entire path, if there are `2$\ast$Z' locations along the path, every
successive pair of which implies a zero coverage zone, then

\[ \mathbf{\forall i \in \left\{1,\ldots,Z\right\}:\;\|P_{z_{2\cdot
i}}-P_{z_{(2\cdot i-1)}}\| \leq d_{zero}} \]

Hence, once again, from our view, it is of high importance that such a
constraint on route design be enforced.

\section{Cost Modeling}
As mentioned earlier, we are interested in simultaneously optimizing for
two scalar costs. The costs pertain to distance and communication costs. In
this report, we have chosen to model only distance cost, and posed the
communication cost as an upper-bounded constraint rather. Hence modeling of
the communication cost is presented elsewhere, in a companion report on
biobjective path planning.

%The imaging cost will depend on the minimum resolution required for image
%capture, as well as queue size (may require to slowdown video acquisition
%while focusing on emptying a near-full transmission queue).

\subsection{Distance Cost}

This cost is a \textbf{mandatory} cost component in all route planning
algorithms that have arisen till date. It represents the length, in
Euclidean space, of the optimal path undertaken by the vehicle. As is
obvious, this cost directly correlates to the limited battery power of the
UAV, which are not fuel-driven. At a certain speed of the vehicle, against
an assumed constant wind speed, the energy of a charged battery will last
only a certain distance. To try maximize the mission, given that
application required tens of kilometers to be remotely surveyed, the
distance cost has to be minimized. In a discrete graph model of the
navigation space, this reduces to finding the shortest path in Euclidean
space (Euclidean shortest path). Such problems are fundamental to the area
of Computational Geometry!

\begin{figure}[!h]
\begin{center}
\includegraphics[scale=.4,keepaspectratio=true]{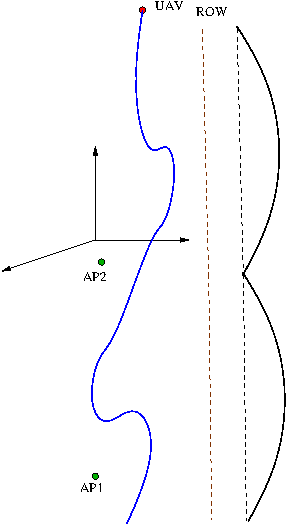}
\end{center}
\caption{A Model for Distance Cost in 2D space. AP represents Communication
Access Point.}
\label{alg_fig}
%\vspace{-.23in}
\end{figure}

For any path in discrete grid, piecewise-linear path assumption is valid
assumption, also because of $L_{min}$ constraint. In such case, insertion
of additional variables to denote the UAV turning points, and sum of
Euclidean distances between such successive turning points, is a direct
measure of such cost.

Let $P_{T_i}, i \in 1,\ldots,T$ be the coordinate tuple of each
turning point along the path in some coordinate system, e.g. world
coordinate system (WCS). Then,
\[\mbox{Path Cost}\;=\;\sum_{i=1}^{T-1}\|P_{T_{i+1}} - P_{T_i}\| \]
where $\|\cdot\|$ denotes an appropriate (e.g., L2) norm.

Insertion of additional variables this way implies we also need to add
\textit{consistency constraints}. Observe that the UAV coordinates $\langle
x_U,y_U,z_U\rangle$, do coincide with the coordinates of turning points at
certain points in time.  Between successive turning points, the UAV
coordinates lie on a straight line.  If certain pair of successive turning
points is within the corridor of flying spanned by ROW (minimum separation)
and CDOF (maximum separation) constraints, then all points along the line
joining this pair will be obeying these twin constraints automatically. If
ALL pairs of successive turning points satisfy ROW and CFOD constraints,
the entire path, define as a sequence of UAV position coordinates, will
satisfy these twin constraints.  Hence we can eliminate variables $x_U$,
$y_U$ and work instead with only coordinates of turning points, $P_{T_i}$,
to compute this cost component.

\section{Final Problem Model}
\label{model_sec}
\textbf{minimize}\\~\\
\(
\mbox{~~~~~~~~~}\left[
\begin{array}{ll}
\mbox{(\textbf{Mandatory})} & \sum\limits_{i=1}^{T-1}\|P_{T_{i+1}} - P_{T_i}\|\\
\mbox{(Not used here)} & {\large \oint_P}
\mathbf{\mbox{~~min}\left(d_{ap1}^\alpha,\cdots,d_{apC}^\alpha\right)\cdot dp}
\end{array}
\right]
\) \\~\\
\textbf{subject to}
\begin{eqnarray*}
%\sum_{i=0}^{L}\|P_i\| \leq L_{max} \mbox{~where $P_i$ = $i^{th}$ turning point
%of flight path}  \\
%\mbox{(total distance traveled)} \\
\mbox{~~~~~~~~~~~}\|P_{T_{i+1}} - P_{T_i}\| \geq L_{min}
\mbox{~~~~~~~~~~~~~~~~~~~~~~~~~~~~~~~~~~~~~~~~~~~~~~~~~~~~~~~~~} \\
\forall i \in \left\{1,\ldots,T\right\}, \forall j \in
\left\{1,\ldots,W\right\} \mbox{~~~~~~~}
y_{P_{T_i}} \geq m_j\cdot x_{P_{T_i}} + c_{j,ROW} \\
\forall i \in \left\{1,\ldots,T\right\}, \forall j \in
\left\{1,\ldots,W\right\} \mbox{~~~~~~}
y_{P_{T_i}} \leq m_j\cdot x_{P_{T_i}} + c_{j,CFOD} \\
\tan^{-1}\left(\frac{y_{P_{T_i}} - y_{P_{T_{i-1}}}}{x_{P_{T_i}} - x_{P_{T_{i-1}}}}\right) - 
\tan^{-1}\left(\frac{y_{P_{T_{i-1}}} - y_{P_{T_{i-2}}}}{x_{P_{T_{i-1}}} - x_{P_{T_{i-2}}}}\right)
\leq 
\theta_{Bmax} \mbox{~~~~} \\
\forall i \in \left\{1,\ldots,Z\right\}:\;\|P_{z_{2\cdot
i}}-P_{z_{(2\cdot i-1)}}\| \leq d_{zero}
\mbox{~~~~~~~~~~~~~~~~~~~~~~~~~~~~~} \\
\end{eqnarray*}
(Note: Some of the $P_{T_i}$ and $P_{z_i}$ may coincide in the final
solution)
~\\

In the above model, $P_{T_i}, i \in 1,\ldots,T$ are the coordinate tuple of each
turning point along the path in some coordinate system, e.g. world
coordinate system. Also, $P_{z_{2\cdot i}}$ and $P_{z_{2\cdot i - 1}}$ are
successively paired entry and exit points of the desired path in each
coverage hole. \textit{ROW} stands for right of way constraint, and DOF
stands for the camera depth of field requirement. The last constraint
models the on-board storage limit in a zero-coverage zone (coverage hole).

\chapter{Evolutionary Background of Shortest Path Solutions}

\section{Dijkstra's Algorithm}
Dijkstra's algorithm, along with Bellman-Ford algorithm, is one of the
earliest popular solution approach to shortest path problems over discrete
graphs. The current popular variant has a fixed single node as the
''source" node, and it finds shortest paths from such source node to
\textit{all other nodes} in the graph, thus producing a shortest-path tree.

Over a graph which has cycles of non-negative cumulative weights, it can be
shown that the shortest path problem is in complexity class ``\textbf{P}".
Hence optimization techniques such as greedy approach can be used to
generate a solution. Dijkstra's algorithm does exactly that. It starts with
source node being the initial node. We presume a notation/definition that
the distance of an arbitrary node {X} be the distance from the initial node
to X. Dijkstra's algorithm assigns some initial distance values to all
nodes, and over multiple iterations, tries to greedily improve, in a
\textit{breadth-first manner}, the distance
values until they converge. The sketch of steps is as follows.

\begin{enumerate}
\item In the initialization step, every node is set to a tentative cost.
Most common initialization is zero for initial node, and infinity for all
other nodes.
\item First, initial node is marked as \textit{current}, while all other
nodes are marked as \textit{unvisited}. A set of all the unvisited nodes is
created, which may be called the unvisited set.
\item For the current node, all of its unvisited neighbors are checked.
Fresh values for their (distance) cost is calculated based on the cost of
the partial path till the current node, plus the weight of the edge from
the current node to the neighboring node being considered. Such newly
calculated tentative cost is compared to the current assigned cost value
for each neighbor, and the smaller of the two is assigned as the new cost
value for the neighboring node.
\item When  all of the neighbors of the current node have been considered,
current node is marked/moved to the \textit{visited} set of nodes, from the
\textit{unvisited} set.
\item Assuming connected graph, which is true in our case, if in any
iteration, the destination node has been marked visited , then the
algorithm is halted.
\item Otherwise, from the \textit{unvisited} set, a node that is marked
with the smallest tentative cost, is picked up as the new ``current node".
Then for next iteration, the algorithm goes back to step (3).
\end{enumerate}

\section{A* Algorithm}
Dijkstra’s algorithm is guaranteed to find a shortest path from the
starting point to the goal, as long as none of the edges have a negative
cost. However, especially in many real-life applications, there are
navigation graphs that contains subgraphs which are \textit{forbidden}.
These forbidden \textit{regions} typically model some kind of obstacles.
Any shortest path algorithm must be able to take this into account and be
able to construct a path that circumnavigates the obstacles and yet produce
a shortest cost path.

It has been shown that in certain scenarios, an example being when
\textit{concave} obstacle is present in the direct shortest way from source to
goal, Dijkstra’s algorithm takes extra time (in terms of amount of partial
paths that could have been shortest towards the goal but are killed due to
presence of obstacle next to them) to provide the guaranteed solution.

At the other end of the spectrum, one can use \textbf{best-first} local
search within the greedy iterations, to improve the cost estimate of the
neighboring nodes. For that, such algorithm keeps some estimate (called a
\textit{heuristic} of how far from the goal any vertex is. Instead of
selecting the vertex closest to the starting point, as is done in
last step of Dijkstra's algorithm, it selects the vertex closest to the
goal. Expectedly, this approach is not guaranteed to find a shortest path.
However, it runs much quicker because it an approximate measure called the
\textit{heuristic function} to guide its way towards the goal quickly.

The trouble is this extremal approach is that the algorithm tries to move
towards the goal \textbf{even} if it’s not the right path. Since it only
considers the cost to get to the goal and ignores the cost of the path so
far, it keeps going even if the path it’s on has become really long.

To avoid weak points of both the extremal approaches, A* algorithm was
developed in 1968, as a combination of both at local level (path
augmentation)\cite{a*_ref_pap}. It quantitatively generates the possibilities that a certain
direction of path extension along a certain neighbor of the current node is
most promising among all neighbors. The algorithm focuses on single-source
single-destination, rather than single-source all-destination shortest path
of Dijkstra's. The construction of the shortest path proceeds via the most promising neighbor. At the same time, other neighbors
are kept on wait list, so that if the current extending partial path hits a
dead end, another promising neighbor with least cost can be picked up a
next most promising extension direction, and extension can proceed along
with that.

The cost function is modeled as
\[ f ( n ) = g ( n ) + h ( n ) \]
Where $g(n)$ is the actual cost from the source node to the current,
intermediate node $n$, and $h(n)$ is the estimated cost from the current
node $n$ to the goal node. $h(n)$ represents a heuristic estimate. A* finds
the best path in shorter time, on an average case, provided that the
heuristic function $h(n)$ obeys certain conditions.  At each step in the A*
path extension, the minimized $f(n)$ value is selected and acted along.

Starting with the source node, the algorithm maintains a priority queue of nodes to
be traversed. The lower the value of $f(n)$ for a given node $n$, the
higher its priority. At each step of the algorithm, the node with the
lowest $f(n)$ value is removed from the queue, the f and g values of
its neighbors are updated accordingly, and these neighbors are added to the
queue. Note that it is \textbf{not necessary} the selected/removed node is a
neighbor of the current node. The algorithm continues until the goal node has a lower f value than
any node in the queue. It is possible that the goal node is visited
multiple times, since some other partial paths may lead to shorter path to
the goal node.

\subsection{Admissibility Criterion}
It has been proved that if the actual cost from \textit{any node} $n$ to
the goal node is $\geq$ estimate $h(n)$, then the A* algorithm is
guaranteed to produce a minimum cost path. This condition is known as the
\textit{admissibility condition}.

\section{Theta* Algorithm}

Due to its simplicity and optimality guarantees under admissibility
criterion, A* has been, till recently, almost always the search method of
choice in real applications. However, in applications in which the
environment is continuous (e.g., the UAV navigation in 3D terrain in our
case), a shortest path found using A* on any discretized graph is
\textbf{not equivalent} to the actual shortest path in the continuous
environment. This is because A* constrains paths to be formed by edges of
the graph which artificially constrains path headings. The twin figures,
borrowed from \cite{theta*_pap}, showcase this gap.

\begin{figure*}[!h]
\begin{center}
\subfloat[Discrete Shortest Path using Square
Grid]{\label{in_1}\includegraphics[scale=.66]{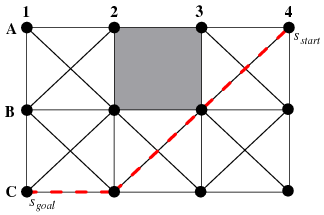}}
\qquad
\subfloat[Actual Shortest Path in Continuous
Environment]{\label{ms_1}\includegraphics[scale=.67]{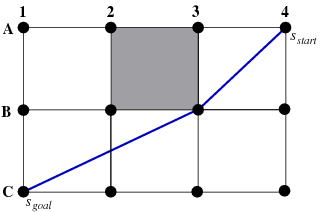}}
\caption{Offset between Discrete and Continuous Shortest Paths, courtesy
\cite{theta*_pap}}
\label{theta_fig}
\end{center}
\end{figure*}

Many solutions have been proposed for this problem, most of which deal with
post-processing of discrete path, to produce a smoother path. However,
expectedly, each post-processing technique produces unrealistic results in
certain space and/or grid configurations. A recent and quite popular
technique which smoothens the path as it is being greedily evolved using A*
algorithm, is known as \textbf{Theta*}. It is also known as the Any-Angle
Path Planning approach for smoother paths in continuous environments
\cite{theta*_pap}.

Theta* is a variant of A* that does not constrain the path being evolved to
stick to graph edges. This in turn allows one to try locate any-angle path
in the unconstrained continuous environment. It is claimed that the pseudo
code for Theta* has only four more lines than the pseudo code for A* and
has a similar runtime, but finds path that is nearly identical to the
continuous shortest path. The key difference between Theta* and A* is that
Theta* allows the parent of a vertex to be any vertex, unlike A* where the
parent must be a visible neighbor. To have such extended, transitive
parental relation, Theta* updates the g-value and parent of an
\textit{unexpanded} visible neighbor s' of vertex s by considering the
following two paths.

\begin{description}
\item[Path 1] As done by A*, Theta* considers the path from the start vertex
to s [= g(s)] and from s to s' in a straight line [= c(s,s')], resulting in
a length of g(s) + c(s,s').
\item[Path 2] To allow for any-angle paths, Theta* additional considers the
\textit{grandfatherly} path
from the start vertex to parent(s) [= g(parent(s))] and from parent(s) to
s' in a straight line [= c(parent(s),s')], resulting in a length of
g(parent(s)) + c(parent(s),s') if s' has line-of-sight to parent(s). The
idea behind considering this alternative is that in case of Euclidean
shortest path problems, this alternative is no longer than previous default
path, due to the \textit{triangle inequality}, if s' has line-of-sight to parent(s).
\end{description}

\begin{figure}
\begin{center}
\includegraphics[scale=.8]{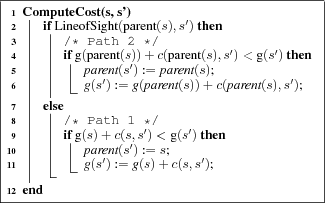}
\end{center}
\caption{Pseudo code for Theta*, courtesy \cite{theta*_pap}}
\label{ps_theta_fig}
%\vspace{-.23in}
\end{figure}

The additional search logic over A* algorithm is shown in Figure
\ref{ps_theta_fig}. We will be using further variation over Theta*
algorithm to design an algorithmic solution for our problem.

\subsection{Variants of Theta*}
Due to immense popularity of Theta*, many variations of it have been
devised till date. Prominent among them are Pi* \cite{incr_pi*_pap}, Lazy
Theta* \cite{lazy_theta_pap}, Angle Propagation Theta*, Augmented Lazy
Theta*, Abstract Augmented Lazy Theta* \cite{aal_theta*_pap} etc. However,
most of them deal with reducing the runtime of the algorithm, not exactly
changing the logic of the algorithm. At this stage, we are focused on
producing better results in terms of path cost offset. Hence we will not be
considering these variations at this stage.

\chapter{Solution for Mono-objective ESP Problem}
\label{mono_sol_chap}
To focus on the constraints and hence the feasible set, we first work out a
restricted problem involving \textit{scalar cost}. We choose the popular
\textbf{path/distance} cost as the cost to first deal with.

While dealing with this cost, all but one constraints remain practically
relevant. The storage constraint, which leads to dealing with zero-coverage
zones, is irrelevant for this cost function. However, retaining this
constraint does not make the solution illegal or invalid. Hence we retain
this constraint here, to understand its impact on the feasible set. In any
practical implementation, the designer is free to drop this constraint if
required.

\section{Configuration of Search Space}
Before working out the solution to an optimization problem, characterizing
the solution is generally done first, to understand what kind of solution/s
one can expect. One of the important characterization deals with whether
the problem is convex. Convexity of an optimization problem in turn implies
convexity of both the underlying set as well as the objective function. These
checks can carried out as described in \cite{nocedal_book}.

Coming to a 3D Euclidean space(a set of points in Euclidean
space), with obstacles as holes, it is
straightforward to see that the underlying set is not
convex. To prove this by providing counter-example, take
any two points on two sides of the obstacle, when connected via a
straight line, will pass through the obstacle. In which case, we
cannot claim that points interior to the obstacle and lying on this line are also part of the
point set, and the problem becomes non-convex. Hence there is no
straightforward characterization of tractability of the problem. In fact,
it is proven that for spaces of dimension 3 and above, the problem is
NP-hard \cite{3dhard_pap}. Hence approximation algorithms inspired by
shortest path algorithms in a convex domain, e.g. Dijkstra, have been
researched aplenty.

The Theta* algorithm, which is the basic approximation algorithm that we
use, itself gives no guarantee of optimality of the solution, with
counter-examples provided to illustrate deviation from optimal solution
\cite{theta*_pap}.However, as we show in section \ref{cmplx_sec} in detail,
Theta* has a marked advantage in terms of computational complexity over the
exact algorithms known for solving the 2D ESP problem. Put another way,
Theta* trades off some degree of optimality with a marked improvement in
computational complexity. Hence we continue/proceed working along
their approach. Their approach is also lucrative from the point of view
that it allows us to plan the path that is lesser constrained in the angle
of turns, thus allowing us to choose within the feasible set with greater
freedom for a problem which entails turn angle constraint.

\section{Discretization of Search Space}
We have already discussed in section \ref{discrete_opt_sec} that though the
route-planning problem is a continuous optimization problem, we will try
solve the discretized version of the problem for various reasons. For
discretizing the search space, which happens to be Euclidean, we impose a
grid for sampling the space. The grid can be regular based on some pattern,
or can be irregular as well. Regular grids are most popularly used in
computational geometry, so we will stick to regular grids. Within regular
grids, we stick to the pattern which has identical cells. Even further, we
assume that the cells are square or cuboid, depending on whether the
Euclidean space being considered is 2D or 3D. We call such grid a
\textit{Uniform Cube Grid}. There can be other uniform grids also, for
example, one discussed in \cite{discrete_approx_pap}.

\subsection{Coarseness of Uniform Cube Grid}
\label{grid_unit_sec}
Given the vast area, running into at least tens of kilometers if not
hundreds, that a UAV specific to our application has to cover, we need to
decide on the appropriate (constant) cell size. In a similar application
discussed in \cite{mil_route_pap}, a 200$\times$300 nautical miles area was
partitioned into cells having constant size of 8$\times$8 nautical miles.
The authors claimed that this spacing corresponds to about one minute of
flying time at the typical UAV cruising speed of Mach 0.8. Another reason
that the authors have considered is the coarseness of the terrain data made
available. The authors had used digital terrain elevation data freely
available from the National Geospatial- Intelligence Agency (2004). In that
data, elevations are accurate to within $\pm$30 meters at least 90\%
of the time, and are provided at points on a grid with \textbf{1 km}
spacing. Any grid that is considered hence has to necessarily have
dimension units, each of which along a particular Euclidean direction is
multiple of the unit size of the DEM grid unit size along the same
direction. Thus, for a one minute flight time coarseness, and typical 
cruising speed and climb and dive rates for their UAVs, they approximated
the route-planning grid with vertices that had a two-kilometer horizontal
spacing and 200-meter vertical spacing.

If any other kind of map is overlaid on the DEM model, which has to be
considered during route planning (e.g. a weather model which predicts
position of stormy areas at a particular moment), then the coarseness of
that map also has to be taken into account. A composite \textit{reference}
map must have grid size, each of whose dimension which is the
\textbf{minimum} of all the overlaid maps \cite{haz_pap}. As already
mentioned, the route-planning map is a multiple of such composite reference
map.

Another useful insight into deciding the grid dimension is provided in
\cite{dgraph_srch_pap}. All fixed-wing UAVs have a limit called
\textit{minimum turning radius}. It is imperative the route-planning grid
dimension be greater than the turning radius, so that shorter turns get
disallowed.

Yet another possibility of fixing grid dimension is to use the minimum
route leg length constraint, described in section \ref{min_rleg_sec}. The
grid dimension can simply be taken to be equal to the constant minimum
route leg length. Choice of such grid dimension implies that any straight
line path segment, when Theta* algorithm is used, is constrained to be a
multiple of the form $(1+\delta)$, where $\delta$ is a positive fractional
number. The \textit{limitation} occurs since $\delta$ is constrained to be
from a small finite set of fractional numbers, and cannot be arbitrary
(i.e. belonging to infinite set).

 Finally, one may note that it is possible to take the world coordinate
system (WCS) as the default grid. Sampling of such grid is based on DGPS
and altimeter, respectively.  A DGPS system gives readings of
latitude/longitude with few meters of accuracy. The most common altimeter in
form of barometer, which are also most robust altimeters, give reading with
accuracy in centimeters. However, hinging the dimension of route-planning
grid to such sampling accuracy will lead to very fine grid, something which
will make the corresponding discrete graph extremely large in size. This in
turn will entail very long path search time. Hence it is important the
route-planning grid dimension be decided based on other considerations such
as few illustrated earlier.

\subsection{An Aside on Random Grid}
There are certain disadvantages of using regular grids. One of the
disadvantage of limited turn angles was highlighted by authors of
\cite{theta*_pap}, and solution of that led to solutions that are more
close to the optimal solution. On similar lines, in \cite{3drisk_pap}, the
authors have claimed that if regular grid is used in the Euclidean space,
then the banking angle constraints lead to suboptimal solutions over cost
function.  Hence they suggested that the grid must be irregular, to that
there are at least few bigger cells, in which the approximation error of
the segment of path passing through them is also \textit{relatively}
smaller. One way of generating an irregular grid, as per their suggestion,
is to generate a ``space" of randomly length-varying lines, and concatenate
them appropriately from source to destination. Then they check various
constraints on such concatenations, to drop out infeasible routes. Thus the
size of the graph of feasible paths from source to destination shrinks to a
smaller feasible set. Over such graph, they propose that one must run any
shortest path algorithm with proper non-uniform Euclidean cost
edge-weights. However, the cost of generating a random grid for a large
side graph adds to the computational complexity of the offline planning
software anyway, which is the design tradeoff involved in pursuing this
approach.

\section{Implementing Various Constraints}
To search for feasible solution, we first try to locate the feasible set
within the appropriate 2D/3D Euclidean space. For certain constraints, it
shall be possible to mark out a subspace of the 2D/3D WCS space in which
route is physically present, as a union of navigable regions. More
precisely, since we discretize the WCS Euclidean space into a grid-based
graph, some of the constraints will lead to a bigger feasible set that is a
forest. For the remaining constraints, we then \textbf{greedily} search
within this forest for either the optimal paths, or just the feasible
paths, while simultaneously obeying the remaining constraints during the
greedy path formation via stepwise path extension.

\subsection{Factoring of Feasible Channel Constraints}
There are two constraints that lead to direct pruning of the WCS
route-planning space. These are the constraints related to minimum
separation (c.f. section \ref{min_sep_sec}) and maximum separation (c.f.
section \ref{max_sep_sec}).

\begin{figure}[!h]
\begin{center}
\includegraphics[scale=.3]{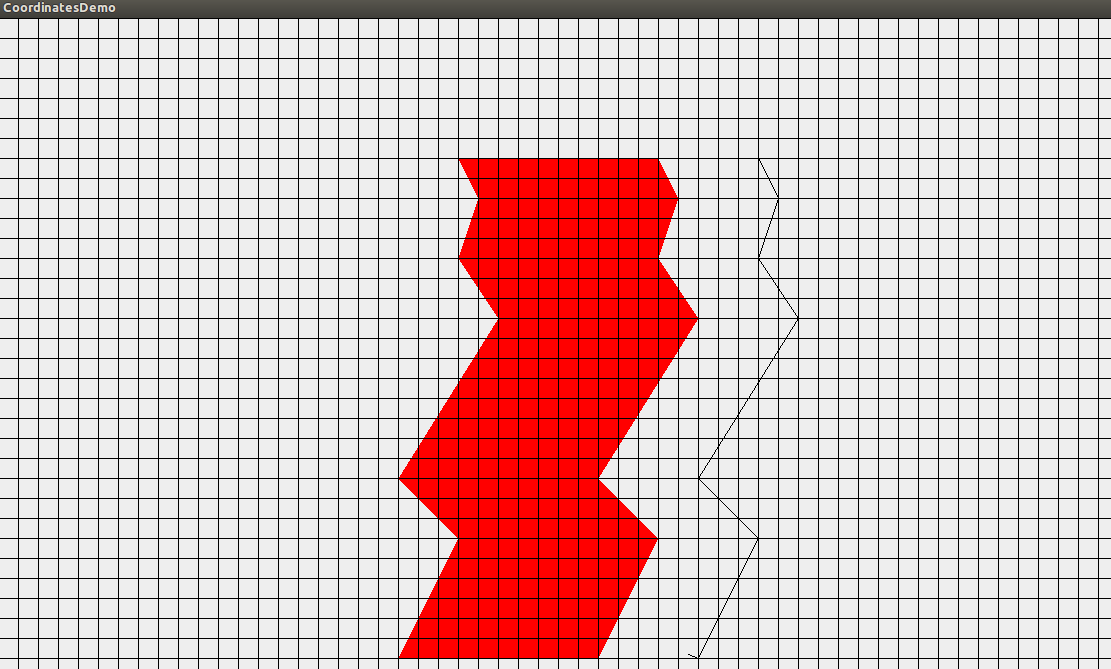}
\end{center}
\caption{A Model of Constrained search space based on separation concerns}
\label{sep_fig}
%\vspace{-.23in}
\end{figure}

\begin{figure*}[!h]
\begin{center}
\subfloat[A Grid by the side of a
Highway]{\label{rl1_1}\includegraphics[scale=.27]{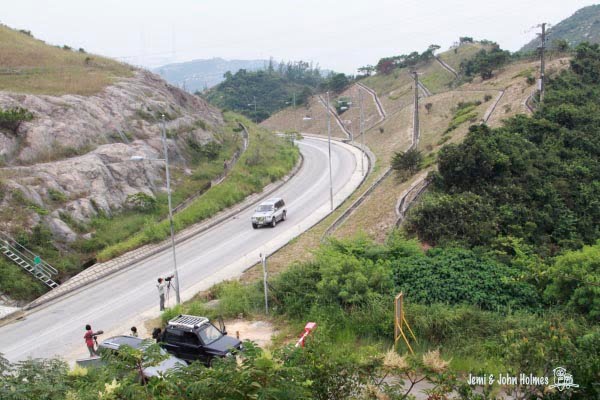}}
\quad
\subfloat[Overlaid Flight Corridor for
UAV]{\label{rl_2}\includegraphics[scale=.27]{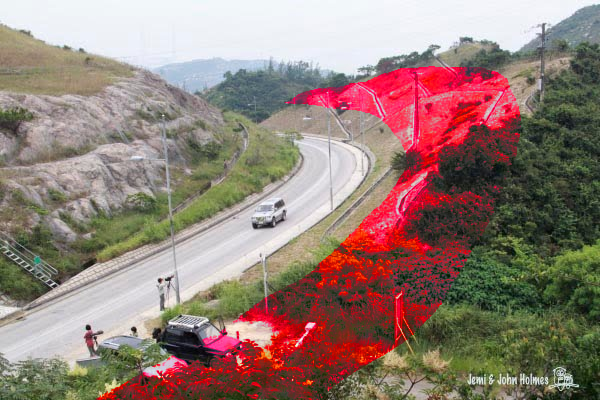}}
\caption{A Realistic Depiction of Flight Corridor. For
brevity, channel is only depicted on one side.}
\label{corr_fig}
\end{center}
\end{figure*}

It is straightforward to see that the imposition of
minimum separation constraint leads to creation of a \textbf{half-space} on
either side of the infrastructure (target of surveillance), that extends
till infinity but does not intersect the infrastructure itself. Similarly,
it is straightforward to see that the imposition of maximum separation
constraint leads to creation of another \textbf{half-space} on either side
of the infrastructure. This half space not only extends in the opposite
direction that of the half-space arising from minimum separation
constraint, but also intersects the infrastructure and does not extend till
infinity.

Since our solution has to be such that both constraints are obeyed
simultaneously, any such solution necessarily has to lie in the
intersection of these two half-space. This in turn means that in 2D case,
such intersection is a corridor each on either side of the infrastructure.
While in 3D case, it means a cylindrical annuli in which the route must be
planned. This pruned search space is depicted in Fig. ~\ref{sep_fig}.

\subsection{Factoring of Minimum Route Leg Length Constraint}
Minimum route leg length constraint can be taken care in multiple ways. One way
is to select a routing grid which has grid dimension same as minimum route
leg length, as discussed earlier. In such case, there is nothing to be
further taken care of, while searching for an optimal solution. If the grid
dimension is not \textit{greater than or equal to}  this constant
constraint, then a modification to Theta* is required, wherever a turn is
encountered.

\begin{figure}[!h]
\begin{center}
\includegraphics[scale=.8,keepaspectratio=true]{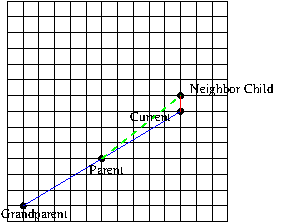}
\end{center}
\caption{Standard Node Nomenclature for Shortest Path Algorithms}
\label{nomen_fig}
%\vspace{-.23in}
\end{figure}

Turn is detected necessarily at each path augmentation step,
to satisfy the turn constraint that is discussed in next
section. To understand how turn is
detected at each step, we follow nomenclature shown in Fig.~\ref{nomen_fig}. We will
also use the following theorem as we design our algorithm.
\begin{thm}
In greedy Theta* shortest path algorithm, the
direction formed by $\langle$grandparent, parent$\rangle$ and the direction
formed by $\langle$parent,
current$\rangle$ are not same. That is to say, the nodes grandparent, parent,
current
are not collinear.
\end{thm}
\begin{proof}
 We first observe that since the evolving path has
already
found a path segment between parent node and current node along the line of
sight between them, there is no obstacle between them that crosses the line
of sight. Similarly, there is no obstacle between grandparent node and
parent
node that crosses their line of sight. So if these three nodes were
collinear, then
Theta* algorithm would have found a direct line of sight between
grandparent
node and current node itself, devoid of any obstacle between them. In which
case, parent of current node would have been the grandparent node. Since
that
is not the case, we have a proof by contradiction that the nodes
grandparent,
parent, current are indeed not collinear, leading to a turn at parent node.
Hence the minimum route leg length would have
already been performed between the grandparent node and parent node, by the
time the algorithm/path augmentation reaches the current node. Therefore,
all
one needs to do is that at the current node, while choosing a neighbor to
augment
the path in best-first way, minimum length check is carried out involving
parent
node, current node and neighbor/child node, only.
\end{proof}

Via this
modification, it is first
checked that if it is possible to directly join the child node and the
parent node, then the distance between them is greater than the minimum
route leg length threshold. If it is not so, then following the
Theta* logic, alternative indirect path augmentation via current node and
then child node may or may not lead to a turning i.e. two possibilities. If
the direction between $\langle$parent, current$\rangle$ and the direction
between $\langle$current, child$\rangle$ are different, then turn is
involved. However, by triangle formed by $\langle$parent, current,
child$\rangle$ nodes need not be an obtuse-angled triangle. Hence one must
not assume and explicitly check for the distance between $\langle$parent,
current$\rangle$ to be greater then the threshold of minimum route length
constraint. In the second case, the possible path segment $\langle$parent,
current, child$\rangle$ are collinear and hence the constraint must be
checked between child node and parent node. If the constraint is not satisfied in
the case that is applicable(out of two), then the indirect path
augmentation via the current node also fails. In such case, the specific
child/neighbor is dropped off consideration, and the next neighboring child
must be considered. If none of the neighbors help in obeying the route
length constraint, then the algorithm must backtrack, in same way as in
handling of the limited-turn-angle constraint, explained next.

\subsection{Factoring in Angle Constraint}
\label{angle_solve_sec}
One of the most popular shortest path algorithm remains the A* algorithm,
on a discrete grid. However, especially for route planning with turning
constraints, A* is known to be inefficient \cite{victor_thesis}. This is
because each of the neighbors, along with path can be possibly extended,
leads to turn angle which is constrained to be one of $\langle
\pm45^\circ,\;\pm90^\circ,\;\pm135^\circ\rangle$. This is a highly
restricted set to be useful for our application, as explained next.

One possible way of dealing with turn-constrained routing is to perform
graph transformation as defined in \cite{turn_transform_pap}, and doing a
path search on the transformed graph. But such algorithm is inefficient for
simple reason: in a 2D grid for example, there are typically at maximum 1
neighbor left post transform for each node, since the turn constraint is
generally about 30$^{\circ}$ \cite{turn_2010_pap}, \cite{angle_pplan_pap},
\cite{haz_pap}. On such a sparse graph, the path which is deemed optimal is
actually very suboptimal when compared to the optimal path in the
continuous version of the problem. Hence we look forward for using
any-angle path planning, a group of algorithms proposed under name of
``\textbf{Theta*}" \cite{theta*_pap}. This path search algorithm leads to
countably finite but much larger set of feasible turning angles at various
points of turn along a path.

A known problem with Theta* algorithm is that at times, it outputs a path
that has many turns. A way of minimizing the number of turns is provided in
\cite{fewer_turns_pap}, \cite{cov_pplan_pap}. However, the same paper shows
that such paths can be up to 22\% more long on an average, in 2D case. In
battery-constrained platform such as UAV, such extra cost cannot be
afforded. Hence we stick with Theta* only, for the base framework,
since it gives much better average case performance (lesser approximation
error).

Since there are many scenarios to be considered within modified Theta*, we
provide a separate, next section to explain that.

\section{Angle-Constrained Theta*}
\label{ac_theta_sec}
\subsection{Angle Consistency with Grandparent}
The first modification within Theta*, over A* algorithm, checks for a
line-of-sight condition between the current node's child/neighbor node, and
its parent node. If line of sight exists, then as proved in previous section, due
to triangle inequality, the path extended till the child node will make a
``turn" at the location of the parent node. The two directions/lines which
are involved and subtend a certain turn angle are a) line between current
node's grandparent and the current node's parent, and b) line between
current node's parent and the current node's child/neighbor being
considered. Hence the first check has to make sure that before current
node's parent and child are connected via a direct path, the above angle is
within the limit prescribed by the constraint.  If the turn angle is
0$^{\circ}$, then the constraint is trivially true.

\begin{figure}[!h]
\begin{center}
\includegraphics[scale=.7]{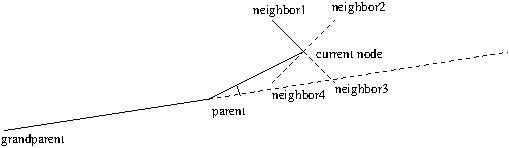}
\end{center}
\caption{Consideration for Turn Constraint in Theta*}
\label{turn_theta_fig}
%\vspace{-.23in}
\end{figure}

\subsection{Angle Consistency with Parent}
If the above angle constraint fails, even if there is a line-of-sight
between parent and child of the current node, then one must fall back to
the A* way in local sense. In such case, one must check whether the angle
subtended between the line from current node's parent to the current node,
and the line from current node to the child being considered, is within the
limit prescribed by the constraint. Note that the line between current
node's parent and current node, due to nature of Theta* algorithm, need not
be aligned at all with any grid dimension or grid diagonals, since the
parent node may be many grid units away from the current node. Hence the
possibility of success at this step is also on the higher side, than the
possibility considering pure A*. Once again, if the turn angle is
0$^{\circ}$, then the constraint is trivially true.

\subsection{Moving on to next Neighbor}
If in both of the above cases (considering or not considering triangle
inequality), the neighbor node being considered is not found
suitable to extend the path along, then we must look for alternatives. In
that case, the child node being checked should be dropped from
consideration, and another neighbor/child node must be looked for possible
and compatible turn constrained path extension.

\subsection{On Exhaustion of all Neighbors}
\label{roll_back_sec}

There are typical scenarios in square grid where plain angle checking
involving grandparent, parent, all other neighbors of current node fails,
especially around obstacles. In such case, only possible option is to
backtrack along the partial path, and try to approach the obstacle using
some other neighbor of a prior node on the backtracked path. This leads to
a possibility whereby the angle of approach towards a obstacle
known/discovered becomes lesser/acute.

Other than obstacles, the algorithm can also fail to extend a partial path
through a coverage-dead zone, which behaves like a \textit{porous
obstacle}. As discussed earlier, only few edges are retained across a
coverage-dead zone, in the navigation graph/grid. In such case, especially
if the approach angle of the partial path towards the boundary of any
coverage-dead zone is close to normal, none of the limited edges across
such zone may subtend an angle with the approaching partial path which is
lesser than the required constrained angle. Hence even if we disregard
proper obstacle modeling since it is not required by our application, we
still need to design for navigation across/around coverage-dead zone. This,
if angle constraints are not getting obeyed, boils down to backtracking and
taking a detour.

\cite{angular_ucv_pap} suggests a heuristics of putting non-uniform weights
to the cell, so that any path approaching an obstacle starts gets
\textit{repulsed} to take a detour around the obstacle.
\cite{angle_pplan_pap} suggests another way, which boils down to having an
adaptive value of cell size. Our way is different since it involves
backtracking in order to circumnavigate the obstacle via alternative route.
The direction of turn is immaterial to us, and hence we do not follow any
approach similar to \cite{beamlet_pap}. Usage of backtracking in the path
evolution is visible in algorithm proposed in \cite{sparse_a*_pap}, albeit
in a different way.

To backtrack, it is not necessary that we backtrack all the way to the
parent node of the current node. In Theta*, it is possible that the parent
node is actually a node very far away from the current node, not
necessarily at a distance of unit cell. It is imperative that the farther
we backtrack and try approach via a detour, the higher the approximation
error become. Hence via a modification to Theta*, we also keep track of
which previous node participated in the triangle inequality for the current
node. We just go back to that specific neighbor of the current node, and
try take a shorter detour when needed. A depiction of such detouring is
shown in Fig.~\ref{detour_fig}. 

\begin{figure}[!h]
\begin{center}
\includegraphics[scale=.9]{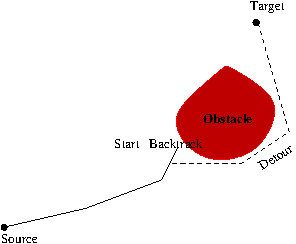}
\end{center}
\caption{Depiction of Unit-step Backtracking and detouring}
\label{detour_fig}
%\vspace{-.23in}
\end{figure}

\subsection{Factoring of Coverage-dead Zone Constraint}
\label{coverage_sec}
This constraint is intimately related to the storage constraint introduced
in section \ref{max_store_sec}. As discussed there, in order to avoid
overflow while storing, one can force the path planning algorithm to only
consider options where the segment of the path lying in a specific
coverage-dead region to have an upper-bounded length.

\begin{figure}[!h]
\begin{center}
\includegraphics[scale=.65]{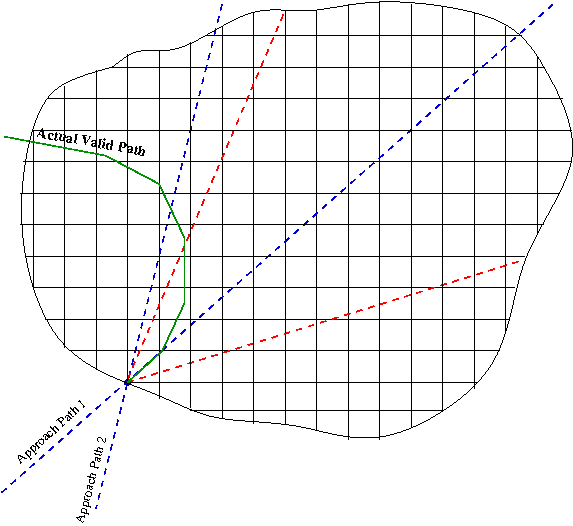}
\end{center}
\caption{Extension of Partial Path using Theta* by Exclusion of Forbidden
Sub-zones within Coverage Restricted Zones}
\label{dead_zone_fig}
%\vspace{-.23in}
\end{figure}

%Any such path segment has to necessarily cross the coverage-dead zone, and
%hence connects two grid points that lie on the boundary of such a zone.
%For any coverage-dead zone, the set of points on its (discrete) boundary at
%the grid intersections is finite-sized and constant (the coverage map is
%static as discussed earlier). Hence, for any specific pair of boundary grid
%points, if they fall on some optimal path from some source point to some target point both outside the region, 

Any such path segment has to necessarily cross the coverage-dead zone. In
case of A*, a valid path segment crossing the dead zone would have
connected two boundary points of the (discretized) dead zone. However, in
case of Theta*, neither of the two points of intersection of a valid path
with the coverage-dead zone boundary may be a grid point. However, since
the path evolves in a greedy fashion, by consideration of one ``best"
neighboring grid point at a time, what is sure is that any partial path
will first intersect any zone's boundary at a grid point only. It may so
happen that after few extensions, the diagonal consideration for path
extension using triangle inequality within Theta* may lead to both the
entry and exit points of the extended path beyond the coverage-dead zone
crosses the zone boundary at fractional grid points. See
Fig.~\ref{dead_zone_fig} for an illustration.

Once a partial path arrives at a grid node on the discrete boundary of a
coverage-dead zone, its extension within dead zone has to be strictly
satisfying the storage constraint i.e. the upper bound on path segment
lying within the dead zone. In a naive design, such upper bounded paths are
expected to be lying in two sectors formed by a chord of such upper bound
length, and the zone boundary. Such sectors are depicted to be to the left
and to the right, respectively, of the two \textit{red-colored} chords of
same upper-bound length, in Fig.~\ref{dead_zone_fig}. A partial path can
arrive at a coverage-dead zone along a direction that when extended can
either fall within one of such sector, or it may fall in the remaining dead
zone. These two options are depicted via the two \textit{blue-colored}
paths shown in Fig.~\ref{dead_zone_fig}.

The presence of such sectors does not rule out the case when a
\textbf{valid} path first moves into the sub-zone corresponding to
breaching of upper bound, then makes consistent turns into safe
sectors/sub-zones and finally exits the dead zone in such a way that
overall, the length of extended path segment within the coverage-dead zone
is within the upper bound. Such situation is depicted by the
\textit{green-colored} path extension shown in Fig.~\ref{dead_zone_fig}.
Hence it is clear that any path extension from inside or on the boundary of
the dead zone, done in a greedy way, must not be done by imposing the upper
bound too early (e.g. via keeping the path extension to be within certain
sectors).

It is also desired that the upper-bounded constraint is not imposed once
the path is \textit{freely} extended in a greedy way within the
coverage-dead zone, and has just exited the zone. For, if such path has
exceeded the upper bound, then one will have to roll back a lot, back to
the point of entry of partial path into the dead zone.  To be able to still
extend path from \textbf{each} current node within the dead zone to one of
its neighbors in a constrained way, one can \textbf{estimate} whether any
such path extension will lead to violation of the storage constraint.

To check for this violation, at the neighboring point of the current node
under consideration, we check if the extension along it will lead to
breaching the length upper bound for the path segment lying strictly within
the coverage-dead zone. This way, we freely extend the path in the
coverage-dead zone, but do not extend it all the way till its exit from the
zone. Before we consider a neighbor, we must make a
shortlist/subset of neighbors of the current node that \textbf{satisfy all
the remaining} constraints. From this shortlist, we consider neighbors
one-by-one for extension. For the current node, upper bound cannot be
breached, else the path extension would not have reached it. However, if
for a particular neighbor, the upper bound is breached, then the path
cannot be extended along that neighbor. In such a case, we consider another
neighbor from the shortlist and try to extend along that neighbor. If we
are unable to extend along any neighbor, then we need to backtrack along
the path evolved so far.

\begin{figure}[!h]
\begin{center}
\includegraphics[scale=.65]{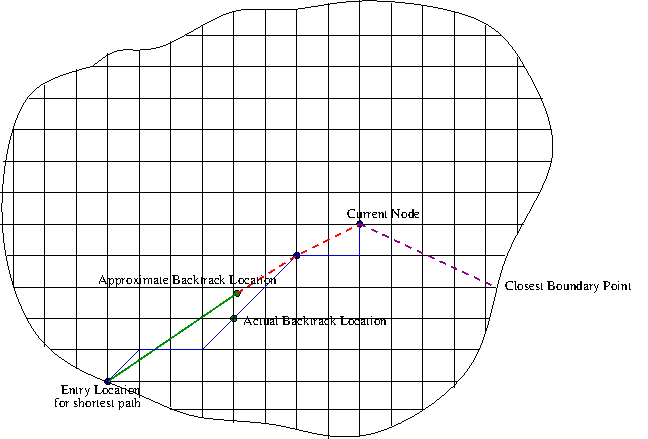}
\end{center}
\caption{Backtracking for Valid Path Extension}
\label{backtrack_fig}
%\vspace{-.23in}
\end{figure}

How much to backtrack? Must one backtrack all the way to the entry point
into the coverage-dead zone, or backtrack just one step back i.e. to the
last neighbor of the current node considered by Theta*, or the parent of
current node? The answer does not lie in these extremes. Instead, it lies
in between. From the current node, one can approximately measure how
\textit{minimally} far one is from the boundary of the coverage-dead zone.
This can be measured by taking the distance of the current node from each
of the finite boundary points, and taking the minimum. Then, along the path
evolved, one must traverse back the same amount of distance, appropriately
rounded off to nearest integer. Since such point need not lie on the grid
corners, one must locate the nearest grid corner \textbf{along} the path of
shortest path evolution. To find such a grid corner, one must consider the
two parent nodes (or the parent node and the current node), between which
the backtracked point has been marked. The sequence of neighbors considered
while evolving path between these two nodes using Theta* is a finite list
of grid points which are expected to be closer to the marked location. We
take individual distances of the marked location of all these grid points,
and pick up the grid point that is closest to marked location. We rescind
the path to such closest grid point. All the edges that were previously
taken to reach the current node using Theta*, are marked as having infinite
weight. From such rescinded grid point, we start our path extension afresh.

For the above backtracking scheme, there will be approximation error in
estimation of path rescinding. However, point to note is that we are trying
to maximally prune the search space based on whatever exploration we do
during the greedy path extension. Also, if we have backtracked enough but
not fully, then at some point, after exploring search space from the
backtracked point, another backtracking will happen. In such incremental
way, we will be able to find a shortest path segment if it exists, if not
in fastest possible manner.

\begin{figure*}[!h]
\begin{center}
\subfloat[Normal Path Evolution without Backtracking]{\label{bt_1}\includegraphics[scale=.54]{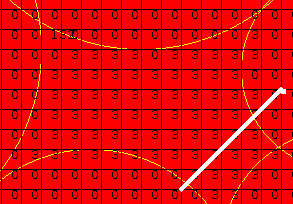}}
\qquad
\subfloat[Evolution of Shorter Path involving Backtracking, closer to
boundary]{\label{bt_2}\includegraphics[scale=.54]{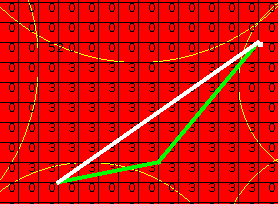}}
\caption{Illustration of Backtracking within Coverage Hole}
\label{back_fig}
\end{center}
\end{figure*}

\subsection{Degenerate Cases during Constrained Expansion}
So far, we provided details of path expansion design, while satisfying 
each of the constraints. However, there are specific path extension 
scenarios, in which the expansion cannot proceed in the straightforward 
way elaborated so far. In such degenerate scenarios, specific design 
choices for expansion are made, trading off ease of design with some 
degree of optimization error.

As first case, it can happen that after finding one better path segment 
about to grow out of a coverage hole, there exist better path segments 
passing through the same coverage hole, via the same entry point for the 
current evolving path. Towards such scenario, whenever a path is entering 
into a coverage hole, we choose to perform a exhaustive path planning 
within the entire non-coverage(hole), an induced subgraph of the path 
planning grid. Even though this increases the computational complexity,
it accounts for  all the possible paths within the same hole that can be 
obtained with the same entry point. This in turn avoids re-exploration of 
the coverage hole in the specific scenario when the head of the currently 
evolved path has grown beyond the hole, and then due to best-first nature 
of the algorithm, a node from inside the coverage hole being considered 
pops up from the heap in the next iteration. If such case happens, then 
some amount of path planning beyond the coverage hole is discarded. To 
avoid such wastage, maximal exploration of a coverage hole is done before 
exiting. This leads to creation of an ''inner`` open list, having the 
coverage hole entry point and its descendants which are inner nodes. While 
expanding using this open list, if the path head reaches a boundary node of 
that coverage hole, path expansion is not allowed to go out of the coverage 
hole until all other possible exit points, given a fixed entry point of 
the planned path into that specific coverage hole, are discovered. At the 
end, only possible exit points remain on the inner open list, and they are 
merged into the main open list, so that the ``best'' possible exit option 
is naturally picked up in the next iteration of (overall) path expansion. 
One must note that while inner exploration of a specific coverage hole is 
ongoing using an inner open list, no other sub-paths from outside the hole 
are allowed to grow (main open list is not considered across iterations). 
Such design serves a purpose that the entry point into the hole remains 
same, a requirement that is elaborated as follows. In summary, while 
exploring a coverage hole for path segments, we do not consider the nodes 
on the evolved path outside the hole, and vice-versa.

As another case, it can happen that a line-of-sight based merging between 
neighbor of the current node that is inside a coverage hole, happens to be 
with parent of the current node that is outside the coverage hole. Such 
merging will lead to added design complexity, since the entry point of the 
evolved path into the coverage hole shifts. In turn, the calculated 
cumulative length so far within coverage hole as described in previous 
section, becomes incorrect. Also, even if this calculation is corrected 
one-time, another such merging at later stage of exploration within the 
coverage hole can lead to re-calculation again. A design option considered 
here is to fix the path entry point into coverage hole fixed, by forcing 
the entry point to be a forced parent node to the current node. Thereafter 
the distance of the current node inside the hole will be checked with 
reference to this "forced" parent only.

As a reverse case, it can happen that a line-of-sight based merging 
between neighbor of the current node that is outside a coverage hole, 
happens to be with parent of the current node that is inside the coverage 
hole. Again, the problem with such merging is that the threshold checking 
done earlier while exploring the coverage hole can get voided, due to 
shift in the exit node and hence change in the calculation of cumulative 
length within hole. The proposed solution is very similar to the solution 
to the previous problem. Here, the exit point of  the hole as a boundary 
point is forced to become a parent node, thus disallowing merging of the 
above nature. Such design also helps in thwarting the possibility of a 
line-of-sight based remote merging between two nodes of an evolving path 
that will intersect a coverage hole that has already been fully explored. 
Such scenario especially arises when there are multiple holes and multiple 
obstacles, as will be most practically the case.

Such design of forcing entry and exit nodes of an evolving path to be 
forced parent node also works in the scenario when an evolving path inside 
a coverage hole just touches the hole boundary, does not go out and 
immediately reenters into the same hole. Even though this increases the 
number of turns in the final path, but it will keep algorithm fail proof 
and less complex.

As a final case, while doing the path planning the
minimum route leg length constraint needs to be checked only when the 
parent, current and neighbor are non-collinear. If they are collinear, we 
know that anyway, it came up to that point satisfying the 
minimum route leg length constraint and addition of one more step in the 
same direction will not make the constrain to fail. Thus we can avoid a lot 
of duplicated checking in the algorithm.

\section{A Note on Migration to Online Path Planning}
In previous chapter, we had argued that an off-line path planning
algorithm is most suitable for the surveillance application that we are
interested in. 
The terrain and the location of the long linear infrastructure being almost
fixed (even in disaster scenario, their components aren't expected to
geolocate themselves significantly), the only reason one would think of
doing incremental path planning is when precomputed connectivity map, or
SNR map, is unavailable. This can happen sometimes, if an area is not
learnt already for its SNR distribution \cite{ghaffar_2010_pap}. At other
times, along the flight, one gets accurate SNR value, rather than estimated
or predicted values, all the time. If the application demands that
communication cost be accurate, biobjective online path planning can be
thought of. In such case, incremental Pi* algorithm, the online version of
basic Theta* algorithm \cite{incr_pi*_pap}, can be modified to meet the
requirements of online path planning. This is in line of A* moving to D*
and its variants \cite{pplan_graph_pap} for generic shortest path problems.

One can also plan for a \textbf{hybrid} approach, whereby a gross
path is designed off-line, and for each coverage hole, based on continuous
SNR observations, path segment within each coverage hole, and till the goal
node possibly beyond, is \textbf{re-planned}, once the vehicle enters a
coverage hole. In such case, backtracking is no more possible, and mission
can fail sometimes. However, if it succeeds, then since the design is based
on real, streaming measurement data, the path length is far less
pessimistic, and hence one ends up saving on the distance cost as well as
communication cost.

\chapter{Scaling Mono-objective ESP Solution to 3D Space}
So far, we have focussed on path planning in 2D domain, for the ease of
understanding. However, in reality, the path of a UAV will be in 3D space.
In practice, the change in altitude of the linear infrastructure will not
be as vast as change in latitude/longitude of the same, especially if it
runs through flat/plateau kind of region. Even then, in theory, the problem
is a 3D Euclidean Shortest Path problem.

In continuous domain, 3D path planning problem is known to be NP-hard
\cite{esp_book}, \cite{3d_heur_pap}. Only heuristic and approximate algorithms exist so far for
solving such continuous problem. The problem becomes more tractable, once
it is discretized. This is one more reason why we turn to discretization of
the problem.

Evaluation of suitability of basic Theta* algorithm for 3D UAV path
planning, which is also the base algorithm for our solution, was done in \cite{3d_pplan_pap}. To reduce certain functional
ambiguities within Theta*, the paper suggests usage of two gain factors in
the cost component functions, $\mathbb{G}$ and $\mathbb{H}$. The paper also introduces a
specific constraint for 3D path: the upper bound on rate of climb
(altitude-gaining) by UAV, at a given (constant) speed (pitch, yaw, roll
something?). More details of this constraint can be found in
\cite{dgraph_srch_pap}.

While working with 3D discrete ESP, one needs to discretize the altitude
direction also. The consideration for discretization of this direction
cannot be the same as that of latitude/longitude, described in section
\ref{grid_unit_sec}. \cite{3d_pplan_pap} suggests using rate of climb,
which is upper-bounded due to capacity of the UAV motors. Once again, in
case of an overlaid map such as DEM being used, the resolution of vertical
axis in that map must also be factored in. 

In case of a 3D discrete grid, any node can have 26 neighbors. Since Theta*
does a best-first based expansion, there are scenarios when either multiple
neighbors of the same current node, or multiple nodes spaced far apart from
each other in the problem space, have same cost value. Which one to choose
in best-first way, so that path extension continues from there in an
average case (worst case, backtracking happens), is a problem more aptly
terms as an \textit{ambiguity} in \cite{3d_pplan_pap}. The problem of
ambiguity becomes acute due to sheer size of local or semi-local
neighborhoods of multiple currents nodes during the iteration for path
evolution. To disambiguate, their authors propose using a weighted
summation of $\mathbb{H}$ and $\mathbb{G}$, instead of plain summation.

Due to too many iterations in the innermost loop (line-of-sight) check for
3D case, basic theta* becomes computation heavy. Hence a lighter version,
more suitable for 3D path planning, called lazy theta*
\cite{lazy_theta_pap} has been worked out.

Since we are considering a unique constraint based on signal strength,
the concept of virtual terrain \cite{3d_heur_pap} can be used to reduce the
problem to 2D. This is because as instantaneous altitude increases for a
UAV, the signal strength is only going to grow weak (we have only ``base"
stations, located on ground). So any optimization algorithm will try pull
the path evolution to lowest altitude (especially the heuristic cost
function part), as long as the path falls within the \textit{3D annulus of
two concentric cylinders}. Hence altitude variation will not be much and 3D
space can be approximated into a 2D virtual terrain. 

\section{Extending Previous Solution}
Many of the strands of previous solution can be scaled directly to the 3D
case. It is easy to see that factoring in of feasible channel constraint in
3D space leads to a corridor which is a contiguous sequence of 3D
cylindrical annulus, which each element/cylindrical annulus in the sequence
representing non-turning/straight segment of the piecewise-linear
representation of the infrastructure itself. To see that, one must just
rotate the illustration in Fig.~\ref{sep_fig} around the medial line of the
winding infrastructure representation. The minimum route leg length constraint
checking scales itself trivially, in the Euclidean space via L2 norm way.

The checking for turn angle constraint between two ``directions" in
3D space implies the angle formed by them in the \textbf{2D-plane} that
they themselves form (any two lines form a plane in Euclidean and many a
non-Euclidean spaces as well). Such an angle can also be decomposed into
two angles, one being formed by the projection of two directions on X-Y
plane, and the other being formed by the projection of the same directions
on either X-Z or Y-Z plane. For most fixed-wing aircrafts, there are
different turn angle upper bounds, especially for the altitude direction,
hence such decomposition can help them being handled properly at the time
of path planning computation. None of the steps for the angle-constrained
Theta* variation that we proposed in section \ref{ac_theta_sec}, need to be
redesigned for 3D case.

The steps involved in dealing with coverage-dead zone also scale well for
3D case. However, most of the steps become computationally heavy, e.g.
finding the minimum distance to the ``surface" (instead of perimeter) of a
3D dead zone. The complexity of this computation, and many others,
increases by a factor of \textbf{n}. This itself can lead to many
backtrackings. Also because of too many ambiguities in choosing the next
neighbor as described earlier, there amount of backtrackings can get
compounded. A better approach than free-extension-then-backtrack must be
designed.

One possibility for that is guided extension within the
coverage-dead zone by biasing the cost towards a specific exit boundary
point, \textbf{during} the period when the path is evolving within the
coverage-dead zone. Such a biasing can ensure that the search for path
extension, including backtrackings, are restricted a segment/portion of the
coverage-dead zone region. Such a cost bias will be over and atop the bias
in path evolution direction that is lent by the heuristic cost component
factor, $h(n)$.

Another possibility is to reduce the computational burden by performing
only the necessary computations in the base algorithm, Theta*. Such an
approach, specifically for 3D path planning, has been evolved by authors of
Theta* itself, and is called \textbf{Lazy Theta*} algorithm. Such an
approach avoids the trap of not covering the entire path search space
within the coverage-dead zone, and yet is computationally lighter than the
previous. For more details, refer \cite{lazy_theta_pap}.

\chapter{Theoretical Properties of Mono-Objective Solution}

Any optimization algorithm, whether exact or approximate, must be checked
for correctness of the solution that it provides, ability to find solution
when it is shown to exist, and the degree of optimality of the solution
generated. In this chapter, we look at each of these properties.

\section{Correctness}
The correctness of a solution to an optimization problem pertains to the
solution obeying all the constraints. For the proof that our optimization
algorithm is correct, one can easily extend Theorem 1 of \cite{theta*_pap},
and the corresponding Lemma 2, to show that it indeed satisfies all the
constraints. We hereby provide the modified statements of the Theorem and
the Lemma 2. Their proofs are straightforwards extension of that provided in
\cite{theta*_pap}.

\begin{lem}
\label{lem1}
At any point during the execution of the modified Basic Theta* for our
problem, following the parents from any vertex in the open or closed lists
to the start vertex retrieves a path that satisfies all the problem
constraints enlisted in section~\ref{model_sec}, from the start vertex to
this vertex in reverse.
\end{lem}

\begin{thm}
When the modified Theta* algorithm terminates, the path extraction
retrieves a path satisfying all the problem constraints, from the start
vertex to the goal vertex if such a path exists.
\end{thm}

The proof for termination of algorithm remains the same as in
\cite{theta*_pap}.  The proof for the Lemma can again be made out using
induction hypothesis.  During each iteration when the current node is
extended, and its neighbors expanded, the algorithm ensures that each of
the constraint is explicitly satisfied, for at least one neighbor. If in a
particular iteration, none of the neighbors can satisfy all the constraints
simultaneously, then path extension cannot happen and therefore
step-by-step backtracking happens. Eventually, if at least one such path
exists, then upon termination of the algorithm, due to comprehensive
searching in search space, the algorithm is able to find out one of the
optimal paths.

\section{Completeness}
The completeness of a solution to an optimization problem pertains to the
algorithm's ability to find the solution (a constraint-obeying path in our
case), if one exists. Again, for the proof that our optimization algorithm
is complete, one can easily extend Theorem 1 of \cite{theta*_pap}, and the
corresponding Lemma 1. We hereby provide the modified statements of the
Theorem and the Lemma 2. Yet again, their proofs are straightforwards
extension of that provided in \cite{theta*_pap}.

\begin{lem}
\label{lem2}
If there exists a (piecewise-continuous) path between any two vertices
satisfying all the constraints, then there also exists a grid path
(sequence of nodes considered from open list, each of whose neighbors were
expanded before the node was put on closed list) between
the same two vertices, which was followed/traversed by our modified Theta*
algorithm, across iterations, to arrive at the compliant path.
\end{lem}

\begin{thm}
The modified Theta* algorithm terminates. If a path satisfying the
constraints is known to exist, the algorithm reports one such path upon
termination. Else, on termination, the algorithm reports that no such path
exists.
\end{thm}

The proof for termination of algorithm remains the same as in
\cite{theta*_pap}. The proof of lemma is also straightforward, since in
each iteration, we consider a current node and its neighbors. Even if we
make a shortcut between the parent of the current node and a neighbor, it
can be shown by local induction on the size of the vertex set, that this
shortcut can be \textit{represented} by a sequence of edge mergers. A graph
edge necessarily runs between two nodes of a graph. Hence the sequence of
edge mergers from one node to another node, leading to formation of an
any-angle path, gets \textit{represented} by a sequence of node traversals.

In the final step to prove the modified theorem, if the algorithm has
terminated when its goal vertex has been expanded, then we can use the
earlier Lemma~\ref{lem1} to show that the algorithm can output one optimal
path. Else, if our algorithm terminates because its open list is empty,
then we need to prove that there does not exist any path from the start
vertex to the goal vertex, that satisfies all the problem constraints. To
prove this statement, we disprove
the following \textbf{contrapositive statement}:\\
If there exists a path satisfying all the problem constraints, from the
start vertex to the goal vertex, then our algorithm cannot terminate on the
specific condition of its open list of vertices being empty.

The proof of above statement follows the same lines of inductive reasoning,
as in \cite{theta*_pap}. Further, it uses Lemma~\ref{lem2} in the proof.
Details of the proof are omitted here for sake of brevity.

\section{Optimality}
\label{opt_sec}
The algorithm we have designed is based on Theta* algorithm. As elaborated
in \cite{theta*_pap}, Theta* algorithm itself is not an optimal algorithm.
It is proved in \cite{lazy_theta_pap} that any optimal \textbf{discrete
shortest path} algorithm, i.e. formed by the edges of 8-neighbor 2D grids
can be up to $\approx$ 8\% longer than the shortest paths in the continuous
environment, in the \textit{worst case}. Similarly, it is shown that the
shortest paths formed by the edges of 26-neighbor 3D grids can be $\approx$
13\% longer than the shortest paths in the continuous environment. The
corresponding proofs are non-constructive, thus only giving the bounds but
not providing the actual method of constructing such discrete ESPs from
scratch. In 2D case, since one can exactly find a continuous ESP using wave
propagation algorithm, one can indirectly construct such discrete ESPs as
well. For both 2D and 3D scenarios, it is experimentally been shown
\cite{anyangle_comp_pap} that in the \textbf{average case}, Theta*
typically finds shorter paths than the worst case.

\section{Computational Complexity}
\label{cmplx_sec}
It is known that Dijkstra's algorithm has a computational complexity of
O(E$\cdot\log$V). For a rectangular mesh, if the the number of nodes,
|V| = $n$, then the complexity becomes $\equiv$ O($n^2\log n$). However,
Dijstra's algorithm computes single source all-target shortest path, i.e.
shortest paths to all $n$ nodes. Hence, on an average
\textit{approximately}, the computation it performs to find single-source
single-target shortest path is $\approx$ O($n\log n$).

On the other hand, the complexity of A$^\ast$ algorithm is known to be
dependent on the heuristic function. In the worst case, the effective
branching factor can be too high, leading to exponential complexity
overall. However, it is also known that the direct distance heuristic, 
using Euclidean/L2 norm, is an \textbf{admissible heuristic}. Further, it
is known that other than being admissible, the direct distance heuristic
leads to expansion of only the nodes that are on the best path, and none
else \cite{amit_note}. In such case, the computation complexity becomes
equal to the lower bound, i.e. O(d), where $d$ is the length/depth of
shortest path. This depth can never be $>$ $n$, and hence the complexity
can be approximated as O(n) in the case of direct distance heuristic.

It is further known that if direct distance heuristic is used, then
A$^\ast$ algorithm provably finds out the optimal path. With this in mind,
we had used direct distance measure as the heuristic when we modified
Theta* algorithm. As is already known, Theta* algorithm is essentially A*
algorithm with one change for any-angle property. It was shown
experimentally in \cite{theta*_pap} that convergence of Theta* is quite
comparable to A*, if not provably. Thus we expect the computational
complexity of our algorithm to be $\mathbf{\approx}$ \textbf{O(n)}.

\subsection{Design Tradeoff}
To highlight the tradeoff involved, let us examine the set of \textit{exact
algorithms} known for path planning in presence of (polygonal) obstacles.
\cite{kapoor_pap}, \cite{hershberger_pap} are some of the seminal papers,
which provide exact, not approximate solutions to the problem. The runtime
of these and many others have been summarized via a table in
\cite{anil_pap}. It can be seen that while the computational complexity of
all these algorithms does not cross O($n\cdot\log n$), it does not fall to
O(n) either i.e. is superlinear in $n$. Similarly, for other non-exact
approaches, especially easy to program meta-heuristic approaches such as
particle swarm and ant colony optimization, relevant texts
\cite{ant_cmplx_pap} suggest that their computational complexity is also
not linear, but O($n\cdot\log n$).

Against this backdrop, the tradeoff
that our algorithm provides becomes clear. We sacrifice a bit on
optimality, as discussed in previous section, while we gain on speed. Since
we deal with huge graphs (corresponding to area of length tens of miles),
the gain in speed is thus significant. The loss in terms of accuracy
however is manageable, rather very less, in an average case. Especially
when the solution is scaled to more realistic 3D space, where problem is
provably NP-hard and only approximate solutions exist, any algorithm with
runtime close to O(n) is expected to have a lot of utility value. This
gives another compelling reason why we chose the specific design option
that led to we choosing to design our approximate, heuristic algorithm. 

\chapter{Performance Benchmarking}
We first experimentally show the
efficiency of our algorithm over contemporary ones, for the additional
approximation error arising out of circumnavigation of obstacles. With such assurance on
practical efficiency of path length, we have more
importantly focussed on another performance: the runtime, or
\textbf{convergence performance}. This is important since path planning
algorithms over \textbf{vast} area, and hence over a mega-sized grid, one must be
able to do fast planning, especially in the cases of disaster management
where an immediate surveillance mission is most preferred. Even
when
the search space is low-dimensional as in our case, the sheer
volume of space entails fast planning efforts.

To the best of our knowledge, no research work exists that considers
coverage holes during path planning as well. Hence it is not possible to 
do
a direct comparison with performance of existing algorithms. We however,
still performed three comparisons.
\section{Performance Degradation due to Incorporation of Extra Constraints}
As mentioned earlier, we modified Theta* algorithm so that greedy path
evolution proceeds while satisfying four extra constraints. Due to extra
computation, it is imperative that the modified algorithm runs slower than
Theta*. In one test, we tried to check, using multiple test cases, how much
slower the algorithm becomes, assuming same source and goal nodes, and same
placement and topology of obstacles (Theta* does not give special treatment
to coverage holes).

It was found that on an average across multiple test cases, the slowdown
was about 2x-3x. On the positive side, despite such slowdown, our algorithm
does provide a more real solution than Theta*, in a more real scenario for
UAV path planning. To recall that Theta* was designed for 2D gaming
scenario, when much of these constraints are not applicable.

\section{Cost Comparison for Single Constraint}
Again, to the best of our knowledge, no research work
exists that combines greedy backtracking with greedy path augmentation (Theta*),
for just the turn-angle constraint only.  Hence a direct comparison with
runtime performance of existing algorithms is at most crude. Still, a short semi-analytical comparison with repulsion-based
angle-constrained path planning \cite{angular_ucv_pap} is discussed below.
\subsection{Repulsion versus Backtracking}
\label{rep_back_sec}
As mentioned before, \cite{angular_ucv_pap} suggests a heuristics of
putting non-uniform weights to the cell, so that any path approaching an
obstacle starts gets \textit{repulsed} to take a detour around the
obstacle. \cite{angle_pplan_pap} implements this suggestion by
providing a formal way to specify this \textit{repulsion ``field"}.
One major disadvantage of this algorithm is that the set of non-uniform weights
around each obstacle is dependent on the topology of the region that the
obstacle represents. For example, for a concave obstacle, due to turn angle
constraint, one will need to model repulsion from a farther distance than a
convex obstacle of same size. Similarly, the size of obstacle is another
parameter that will impact the design of obstacle-specific repulsion
``field" around it. Our algorithm requires no such heavy preprocessing.
Further, our algorithm does efficient,
locally minimal backtracking by design.

Unlike our algorithm,
\cite{angular_ucv_pap} only proposes putting a repulsion ``field" around
each obstacle, and does not guarantee that it is of minimal size.
For a quick numerical comparison, we downloaded and used the same
dataset as used in \cite{angle_pplan_pap}, the \textit{OpenStreetMaps}
(OSM) dataset. We chose to only compare the \textit{relative path length} and
\textit{success rate} as parameters, against the LIAN-5 algorithm therein
(the fastest among the LIAN family).
The relative path length is measured as a fraction against the direct
line-of-sight distance between the source and the goal node. We chose
30 rectangular fragments in which length is few orders more than the width
(size 1950m x 315m). We discretized the grid with a square
cell of size 3m x 3m, as the unit cell. Like \cite{angle_pplan_pap}, we
also marked cells corresponding  to  the  areas  occupied  by  buildings as
un-traversable. We randomly chose source and goal nodes in each fragment
such that the LOS distance is at least 1500m. The angle limit,
$\theta_{Bmax}$, was set to 20$^\circ$. 
As can be seen from table ~\ref{liam_comp_tab}, the discretization error of 6\% is within the limit of 8\%,
established earlier. At the very least, this comparison is able to
point towards \textbf{more closeness} of our approach, to the \textbf{optimal} solution, than a very
recent approach. Also, we are able to find a path in almost all cases,
since our algorithm can backtrack, and hence is able to span most of the
solution space.

\begin{table}[!h]
\begin{center}
\caption{Cost Comparison of LIAN-5 and our Algorithm}
\label{liam_comp_tab}
\setlength{\tabcolsep}{0.5em} % for the horizontal padding
{\renewcommand{\arraystretch}{1.2}% for the vertical padding
\begin{tabular}{|c|c|c|}
\hline
\multirow{2}{*}{~}  & \multicolumn{2}{c|}{$\theta_{Bmax}$ = 20$^\circ$} \\
\cline{2-3}
  & \textbf{success rate} & \textbf{relative path length} \\ \hline \hline
\textbf{LIAN-5} & 87\% & 1.14 \\ \hline
\textbf{Our Algorithm} & 98\% & 1.06 \\ \hline
\end{tabular}
}
\end{center}
\end{table}

\section{Comparing Performance with Brute-force Path Planning}

At one extreme, we have our algorithm that prunes the feasible set of
solutions in a greedy way, and converges fast to an approximate
solution. At the other extreme, we can brute-force enumerate all the
solutions within the feasible set, and then check how many iterations it
takes during the enumeration, \textbf{on an average} across multiple
enumerations/test cases, before we locate an optimal solution. Obviously, a
lot depends on how the enumeration is performed. If we perform the
enumeration in some efficient way, then comparing the performance of our
algorithm to such efficient brute-force solution-finding algorithm leads to
establishment of \textbf{best-case} relative performance of our algorithm.
Such relative performance, in turn, gives a pessimistic lower bound to our
algorithm, and actual performance is expected to be better than that
figure of merit.

\subsection{Algorithm for Brute Force Path Planning}
This algorithm is a small improvement over planning algorithm of \cite{victor_thesis}.
We model our brute-force path planning algorithm on the lines of
intuitive limited backtracking as described in section \ref{coverage_sec}.
We do not perform the most naive way of brute force path planning, that is,
evolving each path independently, end-to-end, and then testing each of them
for being compliant to the conditions of our problem. We avoid doing that
so that we can have a more realistic comparison (average case performance
gain as against worst case performance gain). Similarly, to 
compare one variety of apple with another variety of apple, not 
orange, we use Theta* to evolve brute force path during each iteration of 
brute force algorithm.

We first use basic Theta* to find a path from source node to goal node,
that does not cross any obstacle. Next, we trace the solution path from
source node to goal node, and at each intermediate node, check for all four
constraints that are specific to our problem. If at a certain intermediate
node, if one or more constraints are found to be violated, we start the
search for the next optimal solution in its vicinity. That is, the 
sequence of nodes in current path from source till 
some node in vicinity of 
the specific intermediate node are retained in the same
sequence, certain nodes immediately 
following the intermediate node in the current path
 blocked out, and 
alternative, remaining neighbors to the current node are re-explored to 
find another optimal path in certain scenarios. Which node to 
backtrack to, so that nodes from source till the backtracked node are 
retained, is explained below.

\begin{figure*}[!h]
\begin{center}
\subfloat[Failure 
while previous 
turn is outside]{\label{fail_out}\includegraphics[scale=.38]{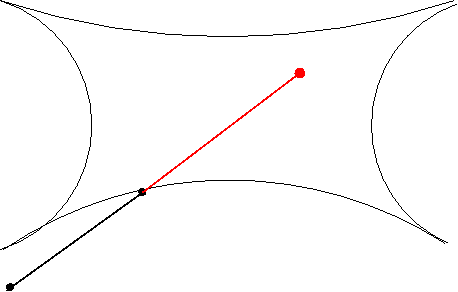}}
\qquad
\subfloat[Failure 
while previous 
turn is inside]{\label{fail_in}\includegraphics[scale=.38]{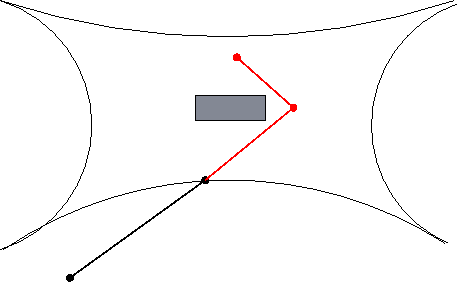}
}
\caption{Failure of Coverage Hole Storage Constraint}
\label{storage_fail_fig}
\end{center}
\end{figure*}

To locate a node where one of the multiple constraints is 
failing, we need to test for compliance of 
each constraint at each node in the current path, in some priority order.
The priority order, chosen intuitively in decreasing 
fashion, is maximum storage constraint (c.f. section~\ref{max_store_sec}) 
followed by maximum turning angle constraint (c.f. 
section~\ref{max_tangle_sec}) followed by minimum route length constraint 
(c.f. section~\ref{min_rleg_sec}). The first check, that of storage 
constraint within coverage hole, can fail in two 
ways, as discussed in the Fig.~\ref{storage_fail_fig}. In the first 
case, there exists no turn between the node of failure for storage 
constraint, and the node for entry into that specific hole, in the 
current path. In the second case, a turn exists between the same two 
above-mentioned nodes. Both of these failure cases need to be treated  
separately. In the first case,
the entire line segment up to the node of failure in the coverage hole 
need to be stored in a separate ``to be blocked'' list, whose reason is 
explained in section~\ref{block_seg_sec}. Further, 
the evolution of next path backtracks to the entry node on boundary of the current coverage 
hole, and resumes from there.
In the latter case, the backtracking needs to happen only up to the 
previous turn (parent node to the node of failure). From there, 
evolution of next path will resume, blocking 
all the nodes on the path segment between the node of failure and its parent.
%The blocking out of
%neighboring node is done so that the same non-conforming path is not
%evolved in the next iteration, using the same neighbor as in the previous
%iteration. The blocking out is done by making the
%corresponding edge weight as infinity.

To start next evolution of path in next 
iteration from a backtracked node, one needs all the information that was 
available in the heap when the same point  was traversed in current 
iteration. It is hence required that the open list of nodes, as it
existed when the current path was evolving while its partial path head
was at the current node, be restored. Only the specific neighbor of the 
node at restoring point that is blocked, has to be removed from the open 
list. Since the open list is heap sorted, it is required to be ensured 
that the sorting order of such list is also restored. To 
achieve that, we store the open list as it exists in the current 
iteration, at each node, so that it can be restored for the next 
iteration if needed. This will increase the memory requirement which 
again emphasis the difficulty to use plain Theta* for these specific 
cases.

\begin{figure}[!h]
\begin{center}
\includegraphics[scale=.6,keepaspectratio=true]{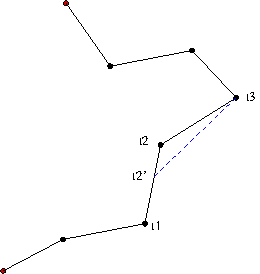}
\end{center}
\caption{Backtracking-based Similar Path Evolution}
\label{brute_bb_path_fig}
%\vspace{-.23in}
\end{figure}

If the current path is 
found to be compliant to the coverage hole storage constraint, then it 
must be checked for the other two constraints as per the priority order. 
Since the action to be taken if at any node, either of the remaining two 
constraint fails, is \textbf{same}, we explain testing for these 
constraints in a \textbf{common} way in the following. Instead of going 
back all the way back to the start node, we can do a local check in the 
vicinity of the the  intermediate node and find any \textit{similar} path 
that may satisfy all the constraints. To find such a similar path, the path in the current 
iteration is 
backtracked to one of the ``last neighbors'' of the node at which the 
constraint failed (referred as \textbf{t2}), one step at a time. From that 
neighbor, it is checked if a path can be evolved to the successor of the 
node at which constraint failed (referred as \textbf{t3}). If we cannot 
find a neighbor and rescind till the parent of \textbf{t2} (referred 
as \textbf{t1}),then we simply forget about \textbf{t3}, block all 
the nodes on the path segment between \textbf{t2} 
and \textbf{t1}, 
and proceed normally to find alternative path. If the algorithm is able to 
find some intermediate node that satisfies all the other constraints 
(minimum route length, maximum angle and line of sight ), we add a path 
segment between this intermediate node and \textbf{t3}, and allow the 
testing of this modified path to proceed from node \textbf{t3} onwards. 
The entire next iteration is depicted in Figure \ref{brute_bb_path_fig}.

\subsubsection{Blocking of Path Segments}
\label{block_seg_sec}
Since we use Theta* as the kernel of brute force algorithm, we also need to
take care of other ways via which the path evolved in current iteration
relapses and coincides with the path evolved and discarded in the previous
iteration. To understand that, one must have a look at
Fig.~\ref{replace_path1_fig} and \ref{replace_path2_fig}.
 
\begin{figure}[!h]
\begin{center}
\includegraphics[scale=.9,keepaspectratio=true]{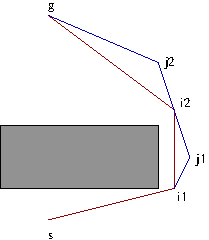}
\end{center}
\caption{Prohibited Brute Force Path Evolution: Case 1}
\label{replace_path1_fig}
%\vspace{-.23in}
\end{figure}

As the brute force path planning for the current iteration evolves from
e.g. node $i_1$ in the Theta* way, it is possible that ends up coinciding
with the previous iterations' path. If that happens, that implies that in
both paths, all segments, including the segment $i_1-i_2$ are common and
overlaid. Since the neighbor of $i_1$ was blocked before the current
iteration started, this perfect overlay is only possible when 
line-of-sight based merging happens. That is, the path of current 
iteration progresses via $i_1-j_2$ and $j_2-i_2$ segments, which in turn 
get triangulated and shorted in lack of any obstacle along segment 
$i_1-i_2$. Lack of obstacle is implied since otherwise, that segment could 
not have been part of the path evolved in previous iteration. To avoid 
such triangulation to happen, one must not
only block the immediate neighbor of $i_1$ in previous iteration, but also
the immediate segment: $i_1-i_2$ as well. In such case, during any future
iterations beyond the previous iterations, such segments cannot be exactly
overlaid on any new path. With respect to Fig.~\ref{replace_path1_fig},
this means that the red-colored segment should be blocked out as well, at
the end of previous iteration. The path is then forced to evolve and
complete as shown in blue color in the same illustration. It is possible
for such blocked red segment to be a proper subset of a longer blue segment
however, because in that case, the enumeration of paths of both iterations,
as successive turning points, will not exactly coincide.

The path coincidence in presence of a coverage hole can be avoided much in
the same way. With respect to Fig.~\ref{replace_path2_fig}, assume that in
the previous iteration, the red-colored path violated the storage
constraint related upper threshold. This in turn implies that the 
path segment of the red path within the coverage hole must not be
retraced in the next iteration's path evolution. Let us assume a convention
to represent such a path segment by the node from which the path evolution
in the previous iteration \textit{entered} into the specific coverage hole.
Hence, in the current iteration, the path will start evolving from node
$e_1$, after one blocks out the neighbor of $e_1$ as found in the previous
iteration. To follow the same blocking mechanism, not only we block the
above-mentioned neighboring node of $e_1$, we also block the path segment
$e_1-i_1$, belonging to the path of the previous iteration. This will
ensure that Theta* does not lead to any triangulation of the kind $e_1-j_1$
and $j_1-i_1$ in the current iteration, thus progressing in the same
way/path as in the previous iteration. Instead, on blocking such
(red-colored) segment $e_1-i_1$, the path evolution in current iteration is
forced to proceed as sequence of non-overlapping $e_1-j_1$, $j_1-j_2$ and
$j_2-g$ segments.

\begin{figure}[!h]
\begin{center}
\includegraphics[scale=.6,keepaspectratio=true]{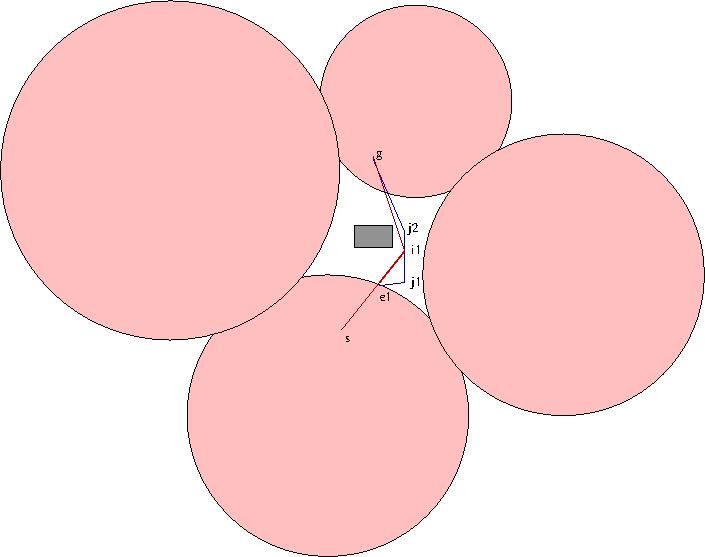}
\end{center}
\caption{Prohibited Brute Force Path Evolution: Case 2}
\label{replace_path2_fig}
%\vspace{-.23in}
\end{figure}

\subsection{Comparison Results}
Comparison of performance of our proposed algorithm was done with brute 
force algorithm, specifically designed for this purpose as explained 
earlier. The comparison was done by executing both algorithm for a 
fixed, same suite of test cases, and observing the set of their execution 
time for the suite. The test cases were designed so as to represent 
different type of terrains and obstacles. Backtracking is inherent 
in both the algorithms and is not deemed to be representing another 
separate iteration. We further assume that the per-iteration time taken 
by both algorithms is similar and comparable, since the distance between 
goal and source node is same for both, and the backtracking strategy is 
similar. In such a case, it is easy to observe that at a coarse level of 
measurement, our algorithm takes just 1 iteration, while the brute force 
testing algorithm was found to take almost 9-10 iterations even for simple 
test cases. For the test cases having at least 2 coverage holes without 
any obstacles, the brute-force algorithm was found to take nearly 50 
iterations. One must note that the coverage holes which we simulated were 
very small (approximately 30 cell size each), in comparison to the size of 
holes that are expected to occur during real deployment.

Generalizing the above observations, it can further be proven that 
whenever the size of the holes is large,the no. of iterations will 
increase significantly. The intuition behind the proof is as follows. 
Figure \ref{bruteIteration_fig} shows the series of paths that are evolved 
over iterations by the brute force algorithm. In that figure, only some 
intermediate paths are drawn for better representation. The dotted black 
lines represent the paths evolved in initial iterations that are not 
satisfying the coverage hole constrain. The green, non-dotted line is the 
path obtained after satisfying the coverage hole constraint, since our 
brute force algorithm is not fully dumb but evolves alternative paths by 
making the current path compliant to constraints in a priority order, 
over successive iterations. For this line, while top priority constraint 
is satisfied, the other constraints are found to fail. Hence in the 
vicinity of the intermediate node, brute force algorithm will search for 
some node that can satisfy other constraints. Final correct path is 
eventually evolved, and is shown in blue color. Given the grid size and 
the no. of permutations involved in taking short detour via neighbors at 
non-compliant intermediate nodes, it is imperative if the grid size is 
large, such sequence of path evolution will take many more iterations to 
converge. Iterations will be much more when the no. of coverage holes and 
the obstacles increase beyond what was the maximum number we were able to 
test with. If we were to compare our algorithm with a \textbf{pure} brute 
force algorithm, in which random paths are generated between the source 
and the goal nodes, tested for compliance of constraints, and having no 
correlation between the path as a sequence of segments between two 
successive iterations, then it is again imperative that the no. of 
iterations taken will be even higher. Hence the time performance of our 
algorithm is better than a brute force baseline algorithm in order of tens.

\begin{figure}[!h]
\begin{center}
\includegraphics[scale=.6,keepaspectratio=true]{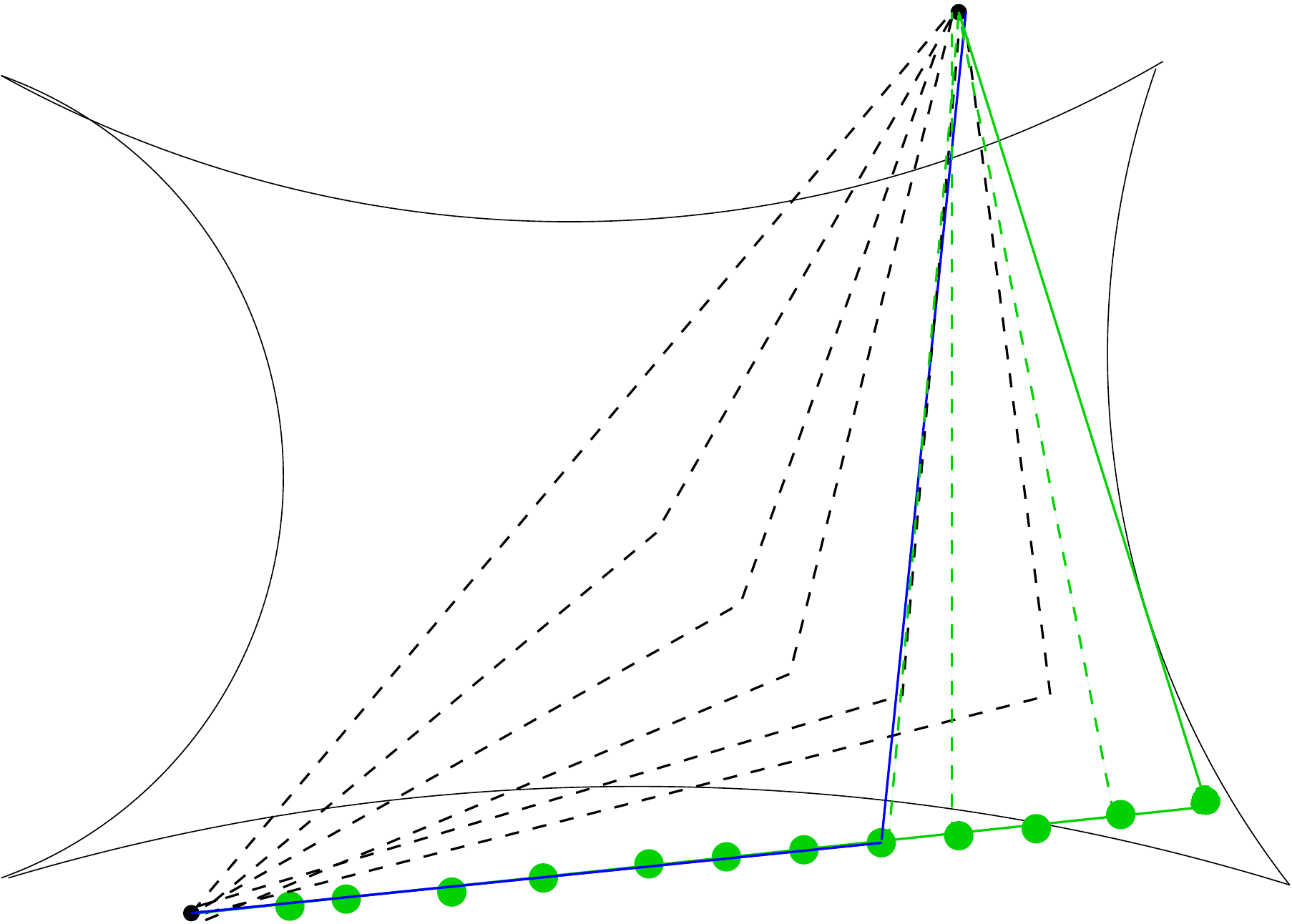}
\end{center}
\caption{Depiction of Iterative Brute Force Evolution of Conformant Path}
\label{bruteIteration_fig}
%\vspace{-.23in}
\end{figure}

\section{Functional Testing}

\begin{figure}[!h]
\begin{center}
\includegraphics[scale=.28,keepaspectratio=true]{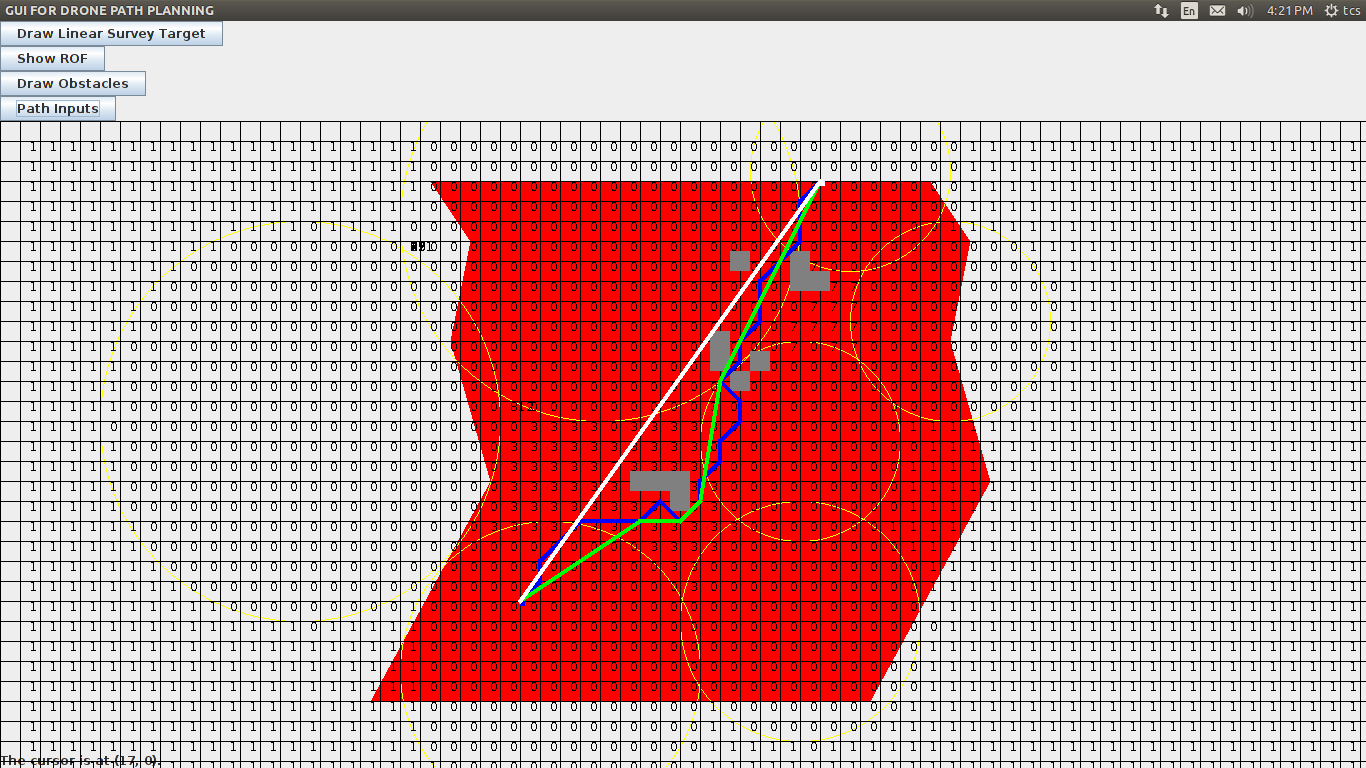}
\end{center}
\caption{Illustration 1 of Functional Test Case}
\label{ill1_fig}
%\vspace{-.23in}
\end{figure}

\begin{figure}[!h]
\begin{center}
\includegraphics[scale=.28,keepaspectratio=true]{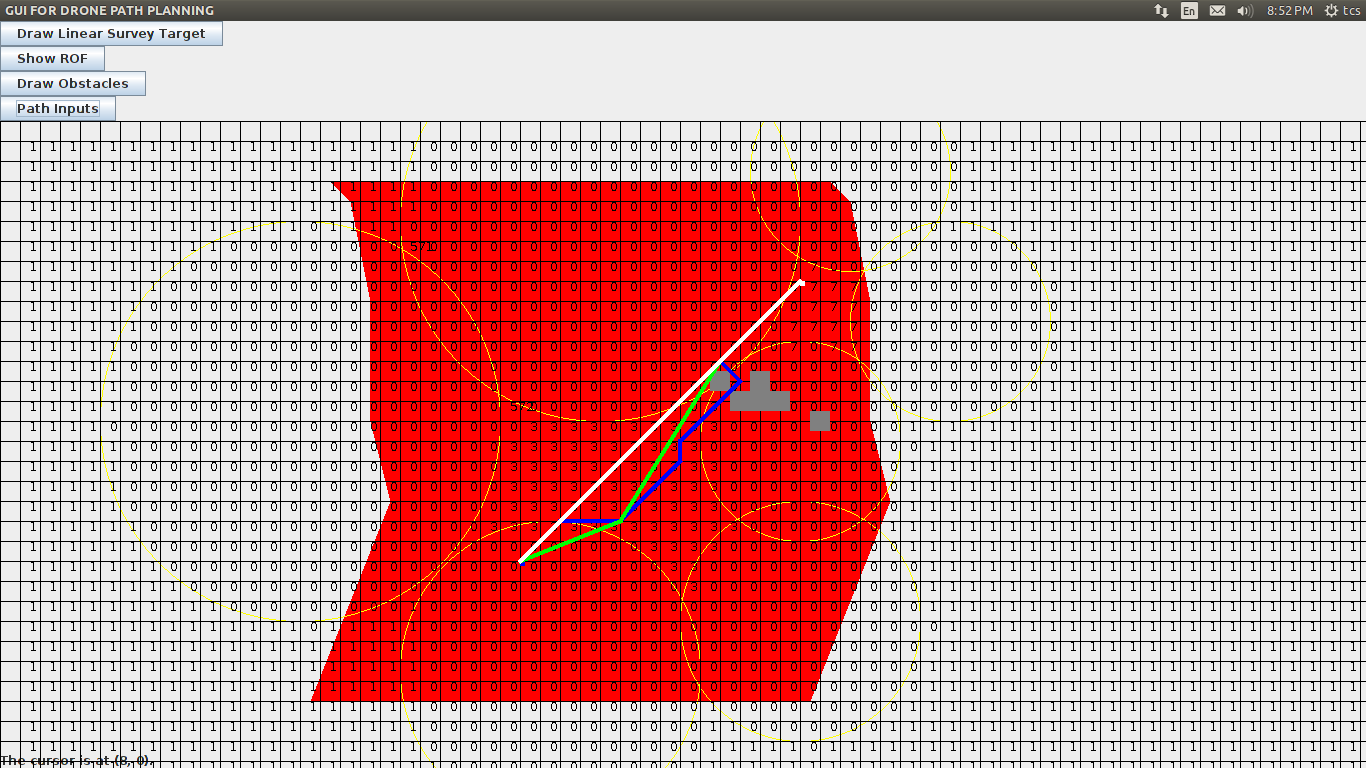}
\end{center}
\caption{Illustration 2 of Functional Test Case}
\label{ill2_fig}
%\vspace{-.23in}
\end{figure}

\section{Simulation Environment}
To be able to evaluate, we
had to design a more comprehensive simulation environment. This simulation
environment is written in Java as a single-threaded application. The simulation
environment encompasses a regular grid-based planning environment. It is
designed to allow placement of
obstacles of any shape and size, as well as coverage holes if
needed, at any
location. Further, they can be placed at any relative geometric configuration and proximities with respect to
one-another. Such freedom entails the obstacles to
be placed partially overlapping, fully ensconced within or fully outside one or
more coverage holes. The simulation environment further allows specification of
source and goal nodes at arbitrary locations, including being
within a hole or outside a hole.

\section{Simulation Results}
The additionally evaluated scenarios included permutations of all of the above-mentioned
degrees of freedom. Though not
exhaustive (around 47 testcases), it did help us in locating many
corner cases and checking our algorithm's readiness to handle those cases.
Illustrative evolution of such paths are depicted in Figs.~\ref{ill1_fig}
and \ref{ill2_fig}. We describe the evolution with respect to the firs t
figure now.
%Illustration of some of these functional tests are shown in few figures next.
In this figure, partial yellow circles depict the coverage regions,
while coverage holes are those regions which aren't part of any circle. The
grey rectangles depict obstacles in various configurations. The
red region depicts the flight planning corridor for a infrastructure that
hypothetically runs along either the right or the left boundary of the
region. The white path depicts the continuous shortest path, the green path
depicts the shortest path evolved by our algorithm, and the blue path
depicts the sequence of nodes that were visited, neighbor after neighbor,
to evolve the final, green path. In the bigger scheme of things,
the white line/direct shortest path is not feasible, since remaining
constraints (discussed earlier in section \ref{model_sec}), including the storage constraint modeling the coverage
requirement, have to be satisfied as well. Many more complex test scenarios were
covered other than the illustrative example here. The broader and
structured list of testcases is available on request.

\chapter{Conclusion}
In the research work so far, we have designed for the maximum communication
scenario while aerial path planning. To the best of our knowledge, such
scenario, though useful, has not been covered in aerial path planning
literature. We have also designed a deterministic backtracking algorithm
for angle constraint in presence of obstacles, and in case of storage
overflow in coverage holes. A design for scaling the solution for 3D path
planning has been devised. One useful thing that we learnt was that the
conversion between a scalar cost and the same being enforced as a
thresholded constraint is sometimes just a theoretical beauty.

Implementing Storage cost as constraint has made our algorithm too
complex, with too many corner cases. While it works, there is always the
possibility that another corner case is recognized up all of sudden on a
future date, and algorithm may need patching up. Hence a fresh
approach for direct biobjective path planning has also been designed.

In future, we are also planning to test our algorithm on Tiger Dataset of
2009. It is a Shortest Path Challenge dataset based on road networks,
provided by NSF(nodes of order of $10^5$ and beyond). This dataset has to
be augmented to contain coverage hole and no-fly zone information, before
we can use it as useful input to our algorithm.

\newpage
\phantomsection \label{bibliography}
\addcontentsline{toc}{chapter}{References}
\bibliography{ref}

\end{document}